\newif\ifllncs
\llncsfalse
\newif\ifshortintro
\shortintrotrue

\ifllncs
  \documentclass[envcountsect,envcountsame,runningheads]{llncs}
\else
\documentclass[10pt]{article}
  \usepackage[margin=1.75in]{geometry}
  \usepackage{times}
  \usepackage{amsthm}
\fi
\usepackage[utf8]{inputenc} 
\usepackage[T1]{fontenc}    
\usepackage{hyperref}       
\usepackage{url}            
\usepackage{nicefrac}       
\usepackage{amsmath,amsfonts,amssymb}
\usepackage{graphicx}
\usepackage{xcolor}
\usepackage{xifthen}
\usepackage{enumitem}
\usepackage{tikz}

\graphicspath{{}{figures/}}

\def\R{\mathbb{R}}
\def\argmin{\mathop{\rm argmin}}
\def\argmax{\mathop{\rm argmax}}
\def\etal.{ \mbox{et al.}}
\def\vs.{\mbox{vs.}}
\def\eg.{\mbox{eg.}}
\def\ie.{\mbox{ie.}}
\def\al.{\mbox{al.}}
\def\mzero{\kern-.1em{\scriptstyle\mathit{0}}\kern-.1em}
\def\tp{{t^\prime}}
\def\fq{f_Q}
\def\fp{f_P}
\def\eff{\textsl{f}}
\def\gammab{\bar{\gamma}}
\def\DP#1#2{\left<{#1},{#2}\right>}
\def\calQ{\mathcal{Q}}
\def\calF{\mathcal{F}}
\def\calC{\mathcal{C}}
\def\calH{\mathcal{H}}
\def\calX{\mathcal{X}}
\def\calY{\mathcal{Y}}
\def\calZ{\mathcal{Z}}
\def\calV{\mathcal{V}}
\def\calE{\mathcal{E}}

\def\calP{\mathcal{P}}
\def\pspace#1{\calP_{\!{#1}}}
\def\calM{\pspace{\calX}} 
\def\subH{_{_\calH}}
\def\pfd{{\raisebox{-.3ex}{\small\#}}}

\newcommand{\E}[2][]{\mathbb{E}\ifthenelse{\equal{#1}{}}{}{_{#1}}\!\left[{#2}\right]}
\newcommand{\grad}[1][]{\mathrm{grad}\ifthenelse{\equal{#1}{}}{}{_{#1}}\,}
\newcommand{\stacktwo}[2]{\stackrel{\scriptstyle #1}{\scriptstyle #2}}

\def\qedsymbol{\hfill\rule[-.2ex]{1.1ex}{1.1ex}}

\newenvironment{proofx}[1]{%
  \smallskip\par\noindent
  \begingroup\small\emph{#1~~}}{%
  \endgroup\par\medskip}
\renewenvironment{proof}{%
  \begin{proofx}{Proof}}{%
  \qedsymbol\end{proofx}}

\ifllncs
\else
  \newtheorem{theorem}{Theorem}[section]
  \newtheorem{lemma}[theorem]{Lemma}
  \newtheorem{proposition}[theorem]{Proposition}
  \newtheorem{corollary}[theorem]{Corollary}
  \newtheorem{definition}[theorem]{Definition}
  \theoremstyle{remark}
  \newtheorem{remark}[theorem]{Remark}
  \newtheorem{example}[theorem]{Example}
\fi

\ifshortintro
\title{Geometrical Insights \\
  for Implicit Generative Modeling}
\else
  \title{Geometrical Insights for Unsupervised Learning}
\fi

\ifllncs
  \author{Leon Bottou\inst{1} \and
    Martin Arjovsky\inst{2} \and
    David Lopez-Paz\inst{3} \and
    Maxime Oquab\inst{4}
  }
  \institute{
    Facebook AI Research, New York, \email{leon@bottou.org} \and
    New York University, New York, \email{martinarjovsky@gmail.com} \and
    Facebook AI Research, Paris, \email{dlp@fb.com} \and
    Inria, Paris, \email{maxime.oquab@inria.fr}
  }
  \else
  \def\ss#1{\,\raisebox{.7ex}{\small #1}}
  \author{Leon Bottou\ss{a,b},
    Martin Arjovsky\ss{b},
    David Lopez-Paz\ss{a},
    Maxime Oquab\ss{a,c}\\[1ex]
    \ss{a} Facebook AI Research, New York, Paris\\
    \ss{b} New York University, New York\\
    \ss{c} Inria, Paris.}
\fi

\begin{document}
\maketitle

\ifshortintro

\begin{abstract}
  Learning algorithms for implicit generative models can optimize a
  variety of criteria that measure how the data distribution differs
  from the implicit model distribution, including the Wasserstein
  distance, the Energy distance, and the Maximum Mean Discrepancy
  criterion. A careful look at the geometries induced by these
  distances on the space of probability measures reveals interesting
  differences. In particular, we can establish surprising approximate
  global convergence guarantees for the $1$-Wasserstein distance, even
  when the parametric generator has a nonconvex parametrization.
\end{abstract}

\section{Introduction}
\label{sec:introduction}
\label{sec:comparingprobas}

Instead of representing the model distribution with a parametric
density function, implicit generative models directly describe how to
draw samples of the model distribution by first drawing a sample $z$
from a fixed random generator and mapping into the data space with a
parametrized generator function $G_\theta(z)$. The reparametrization
trick \cite{challis-barber-2012,rezende-2014}, Variational
Auto-Encoders (VAEs)~\cite{kingma-welling-2013}, and Generative
Adversarial Networks (GANs)~\cite{goodfellow-2014} are recent
instances of this approach.

Many of these authors motivate implicit modeling with the computational
advantage that results from the ability of using the efficient
back-propagation algorithm to update the generator parameters.  In
contrast, our work targets another, more fundamental, advantage of
implicit modeling.

Although unsupervised learning is often formalized as estimating the
data distribution~\cite[\S14.1]{hastie-tibshirani-friedman-2009}, the
practical goal of the learning process rarely consists in recovering
actual probabilities. Instead, the probability models are often
structured in a manner that is interpretable as a physical or causal
model of the data. This is often achieved by defining an interpretable
density $p(y)$ for well chosen latent variables $y$ and letting the
appearance model $p(x|y)$ take the slack. This approach is well
illustrated by the \emph{inverse graphics} approach to computer
vision~\cite{lee-nevatia-2004,kulkarni-2015,romaszko-2017}. Implicit
modeling makes this much simpler:
\begin{itemize}
\item
  The structure of the generator function $G_\theta(z)$ could be
  directly interpreted as a set of equations describing a physical
  or causal model of the data \cite{kocaoglu-2017}.
\item
  There is no need to deal with latent variables, since
  all the variables of interest are explicitly computed
  by the generator function.
\item
  Implicit modeling can easily represent simple phenomena involving a
  small set of observed or inferred variables.  The corresponding
  model distribution cannot be represented with a density function
  because it is supported by a low-dimensional manifold. But
  nothing prevents an implicit model from generating such samples.
  \end{itemize}

Unfortunately, we cannot fully realize these benefits using the
popular Maximum Likelihood Estimation (MLE) approach, which
asymptotically amounts to minimizing the Kullback-Leibler~(KL)
divergence $D_{KL}(Q,P_\theta)$ between the data distribution $Q$ and
the model distribution~$P_\theta$,
\begin{equation}
  \label{eq:kldivergence}
  D_{KL}(Q,P_\theta)
  = \int \log\left(\frac{q(x)}{p_\theta(x)}\right) q(x) d\mu(x)
\end{equation}
where $p_\theta$ and $q$ are the density functions of $P_\theta$ and
$Q$ with respect to a common measure $\mu$. This criterion is
particularly convenient because it enjoys favorable statistical
properties \cite{cramer-1946} and because its optimization can be
written as an expectation with respect to the data distribution,
\[
  \argmin_\theta D_{KL}(Q,P_\theta)
  ~=~ \argmin_\theta \E[x\sim Q]{-\log(p_\theta(x))}
  ~\approx~ \argmax_\theta \prod_{i=1}^{n} p_\theta(x_i)~.
\]
which is readily amenable to computationally attractive stochastic
optimization procedures~\cite{bottou-curtis-nocedal-2016}.  First,
this expression is ill-defined when the model distribution cannot be
represented by a density. Second, if the likelihood $p_\theta(x_i)$ of
a single example $x_i$ is zero, the dataset likelihood is also zero,
and there is nothing to maximize. The typical remedy is to add a noise
term to the model distribution. Virtually all generative models
described in the classical machine learning literature include such a
noise component whose purpose is not to model anything useful, but
merely to make MLE work.

Instead of using ad-hoc noise terms to coerce MLE into optimizing a
different similarity criterion between the data distribution and the
model distribution, we could as well explicitly optimize a different
criterion. Therefore it is crucial to understand how the
selection of a particular criterion will influence the learning
process and its final result.

Section~\ref{sec:adversarialnets} reviews known results establishing
how many interesting distribution comparison criteria can be
expressed in adversarial form, and are amenable to tractable
optimization algorithms. Section~\ref{sec:emvsdisco} reviews the
statistical properties of two interesting families of distribution
distances, namely the family of the Wasserstein distances and the
family containing the Energy Distances and the Maximum Mean
Discrepancies. Although the Wasserstein distances have far worse
statistical properties, experimental evidence shows that it can
deliver better performances in meaningful applicative setups.
Section~\ref{sec:lengthspaces} reviews essential concepts about
geodesic geometry in metric spaces. Section~\ref{sec:geodesics}
shows how different probability distances induce different geodesic
geometries in the space of probability measures. 
Section~\ref{sec:convexity} leverages these geodesic structures to
define various flavors of convexity for parametric families of
generative models, which can be used to prove that a simple gradient descent
algorithm will either reach or approach the global minimum regardless
of the traditional nonconvexity of the parametrization of the model
family. In particular, when one uses implicit generative models,
minimizing the Wasserstein distance with a gradient descent algorithm offers
much better guarantees than minimizing the Energy distance.

\else

\input{longintro.tex}

\fi

\section{The adversarial formulation}
\label{sec:adversarialnets}

The adversarial training framework popularized by the Generative
Adversarial Networks (GANs)~\cite{goodfellow-2014} can be used to
minimize a great variety of probability comparison criteria.  Although
some of these criteria can also be optimized using simpler algorithms,
adversarial training provides a common template that we can use to
compare the criteria themselves.

This section presents the adversarial training framework and reviews
the main categories of probability comparison criteria it supports,
namely Integral Probability Metrics (IPM) (Section \ref{sec:ipm}),
\eff-divergences (Section~\ref{sec:fdivergences}), Wasserstein
distances (WD) (Section~\ref{sec:wasserstein}), and Energy Distances
(ED) or Maximum Mean Discrepancy distances (MMD)
(Section~\ref{sec:ed-mmd}).

\subsection{Setup}
\label{sec:setup}

Although it is intuitively useful to consider that the sample space
$\calX$ is some convex subset of $\R^d$, it is also useful to spell
out more precisely which properties are essential to the development.
In the following, we assume that $\calX$ is a \emph{Polish metric space},
that is, a complete and separable space whose topology is
defined by a distance function
\[
  d:\quad\left\{ \begin{array}{rcl}
    \calX\times\calX & ~\rightarrow~ & \R_+\cup\{+\infty\} \\
    (x,y) & ~\mapsto~ & d(x,y) \end{array}\right.
\]
satisfying the properties of a metric distance:
\begin{equation}
  \label{eq:metricdistance}
  \forall x,y,z\in\calX \quad \left\{ \begin{array}{r@{\quad}l@{\quad}l}
    (\mzero) & d(x,x)= 0 & \text{\small(zero)} \\
    (i) & x\neq y \Rightarrow d(x,y)>0 &\text{\small(separation)} \\
    (ii) & d(x,y)=d(y,x) & \text{\small(symmetry)} \\
    (iii) & d(x,y)\leq d(x,z)+d(z,y) &  \text{\small(triangular inequality)}
   \end{array}\right.
\end{equation}

Let $\mathfrak{U}$ be the Borel $\sigma$-algebra generated by all the
open sets of $\calX$. We use the notation $\calM$ for the set of
probability measures $\mu$ defined on $(\calX,\mathfrak{U})$, and the
notation $\calM^p\subset\calM$ for those
satisfying~$\E[x,y\sim\mu]{d(x,y)^p}{<}\infty$. This condition
is equivalent to $\E[x\sim\mu]{d(x,x_0)^p}{<}\infty$ for an arbitrary
origin $x_0$ when $d$ is finite, symmetric, and satisfies the
triangular inequality.

We are interested in criteria to compare elements of $\calM$,
\[
  D:\quad\left\{ \begin{array}{rcl}
    \calM\times\calM & ~\rightarrow~ & \R_+\cup\{+\infty\} \\
    (Q,P) & ~\mapsto~ & D(Q,P) \end{array}\right.~.
\]
Although it is desirable that $D$ also satisfies the properties of a
distance \eqref{eq:metricdistance}, this is not always possible.  In
this contribution, we strive to only reserve the word \emph{distance}
for criteria that satisfy the properties \eqref{eq:metricdistance} of
a metric distance. We use the word
\emph{pseudodistance}\footnote{\relax Although failing to satisfy the
  separation property (\ref{eq:metricdistance}.$i$) can have
  serious practical consequences, recall that a pseudodistance always
  becomes a full fledged distance on the quotient space
  $\calX/\mathcal{R}$ where $\mathcal{R}$ denotes the equivalence
  relation $x\mathcal{R}y\Leftrightarrow{d(x,y)}{=}0$. All the theory
  applies as long as one never distinguishes two points separated by a
  zero distance.} when a nonnegative criterion fails to satisfy
the separation property (\ref{eq:metricdistance}.$i$).  We use the
word \emph{divergence} for criteria that are not symmetric
(\ref{eq:metricdistance}.$ii$) or fail to satisfy the triangular
inequality (\ref{eq:metricdistance}.$iii$).

We generally assume in this contribution that the distance $d$ defined
on $\calX$ is finite. However we allow probability comparison criteria
to be infinite. When the distributions $Q,P$ do not belong to the
domain for which a particular criterion~$D$ is defined, we take that
$D(Q,P){=}0$ if $Q{=}P$ and $D(Q,P){=}+\infty$ otherwise.

\subsection{Implicit modeling}
\label{sec:implicitmodeling}

We are particularly interested in model distributions $P_\theta$ that
are supported by a low-dimensional manifold in a large ambient sample
space (recall Section~\ref{sec:comparingprobas}). Since such
distributions do not typically have a density function, we cannot
represent the model family $\calF$ using a parametric density
function. Following the example of Variational Auto-Encoders (VAE)
\cite{kingma-welling-2013} and Generative Adversarial Networks (GAN)
\cite{goodfellow-2014}, we represent the model distributions by
defining how to produce samples.

Let $z$ be a random variable with known distribution $\mu_z$ defined
on a suitable probability space $\mathcal{Z}$ and let $G_\theta$ be a
measurable function, called the \emph{generator}, parametrized by
$\theta\in\R^d$,
\[
   G_\theta:~~ z\in\mathcal{Z}~\mapsto~ G_\theta(z)\in\calX~.
\]
The random variable $G_\theta(Z)\in\calX$ follows the
\emph{push-forward} distribution\footnote{\relax
  We use the notation $f\pfd\mu$ or $f(x)\pfd\mu(x)$ to denote the probability
  distribution obtained by applying function $f$ or expression $f(x)$
  to samples of the distribution $\mu$.}
\[
   G_\theta(z)\pfd\mu_Z(z) :~~ A\in\mathfrak{U}~\mapsto \mu_z(G_\theta^{-1}(A))~.
\]
By varying the parameter $\theta$ of the generator $G_\theta$, we can
change this push-forward distribution and hopefully make it close to
the data distribution $Q$ according to the criterion of interest.

This \emph{implicit modeling approach} is useful in two ways.  First,
unlike densities, it can represent distributions confined to a
low-dimensional manifold.  Second, the ability to easily generate samples
is frequently more useful than knowing the numerical value of the density
function (for example in image superresolution or semantic
segmentation when considering the conditional distribution of the
output image given the input image). In general, it is computationally
difficult to generate samples given an arbitrary high-dimensional
density \cite{neal-2001}.

Learning algorithms for implicit models must therefore be formulated
in terms of two sampling oracles. The first oracle returns training
examples, that is, samples from the data distribution $Q$. The second
oracle returns generated examples, that is, samples from the model
distribution~$P_\theta=G_\theta\pfd\mu_Z$. This is particularly easy
when the comparison criterion $D(Q,P_\theta)$ can be expressed in
terms of expectations with respect to the distributions $Q$
or~$P_\theta$.

\subsection{Adversarial training}
\label{sec:adversarialtraining}

We are more specifically interested in distribution comparison
criteria that can be expressed in the form
\begin{equation}
  \label{eq:vform}
  D(Q,P) = \sup_{(\fq,\fp)\in\calQ} \E[Q]{\fq(x)} - \E[P]{\fp(x)}~.
\end{equation}
The set $\calQ$ defines which pairs $(\fq,\fp)$ of real-valued
\emph{critic} functions defined on $\calX$ are considered in this
maximization. As discussed in the following subsections, different
choices of $\calQ$ lead to a broad variety of criteria. This
formulation is a mild generalization of the Integral Probability
Metrics (IPMs) \cite{mueller-1997} for which both functions $\fq$ and
$\fp$ are constrained to be equal (Section~\ref{sec:ipm}).

Finding the optimal generator parameter $\theta^*$ then amounts
to minimizing a cost function $C(\theta)$ which itself
is a supremum,
\begin{equation}
  \label{eq:problem}
  \min_\theta\left\{~
  C(\theta) ~ \stackrel{\Delta}{=} \max_{(\fq,\fp)\in\calQ}
     \E[x\sim Q]{\fq(x)} - \E[z\sim\mu_z]{\fp(G_\theta(z))} ~\right\}~.
\end{equation}
Although it is sometimes possible to reformulate this cost function in
a manner that does not involve a supremum (Section~\ref{sec:ed-mmd}),
many algorithms can be derived from the following variant of the
envelope theorem~\cite{milgrom-segal-2002}.

\begin{theorem}
  \label{th:envelope}
  Let $C$ be the cost function defined in \eqref{eq:problem}
  and let $\theta_0$ be a specific value of the generator parameter.
  Under the following assumptions,
  \begin{itemize}[nosep]
  \item[a.] there is $(\fq^*,\fp^*)\in\calQ$ such that
    $C(\theta_0)=\E[Q]{\fq^*(x)}-\E[\mu_z]{\fp^*(G_{\theta_0}(z))}$,
  \item[b.] the function $C$ is differentiable in $\theta_0$,
  \item[c.] the functions $h_z = \theta\mapsto\fp^*(G_\theta(z))$ are
    $\mu_z$-almost surely differentiable in $\theta_0$,
  \item[d.] and there exists an open neighborhood $\calV$ of $\theta_0$
    and a $\mu_z$-integrable function $D(z)$ such that~$\forall\theta{\in}\calV$,
    $|h_z(\theta)-h_z(\theta_0)|\leq D(z) \|\theta-\theta_0\|$,
  \end{itemize}
  \medskip
  we have the equality $\displaystyle ~
     \grad[\theta] C(\theta_0) = -\E[z\sim\mu_z]{\,\grad[\theta]h_z(\theta_0)\,}~.$
\end{theorem}

This result means that we can compute the gradient of $C(\theta_0)$
without taking into account the way $\fp^*$ changes with
$\theta_0$. The most important assumption here is the
differentiability of the cost $C$. Without this assumption, we can
only assert that $-\E[z\sim\mu_z]{\,\grad[\theta]h_z(\theta_0)\,}$
belongs to the ``local'' subgradient
\[
   \partial^{\text{loc}}C(\theta_0) \stackrel{\Delta}{=}
     \left\{\: g\in\R^d ~:~ \forall \theta\in\R^d~~ C(\theta) \geq
       C(\theta_0)+\DP{g}{\theta{-}\theta_0}+ o(\|\theta{-}\theta_0\|)
     \:\right\}~.
\]   

\medskip

\begin{proof}
  Let $\lambda\in\R_+$ and $u\in\R^d$ be an arbitrary unit vector.
  From~\eqref{eq:vform},
  \begin{align*}
    C(\theta_0+\lambda{u})&\geq
    \E[z\sim Q]{\fq^*(x)}-\E[z\sim\mu_z]{\fp^*(G_{\theta_0+\lambda{u}}(z))}\\
  C(\theta_0+\lambda{u})-C(\theta_0)&\geq
  -\,\E[z\sim\mu_z]{h_z(\theta_0+\lambda u)-h_z(\theta_0)}~.
  \end{align*}
  Dividing this last inequality by $\lambda$, taking its
  limit when $\lambda\rightarrow0$, recalling that the dominated convergence
  theorem and assumption (d) allow us to take the limit inside the
  expectation operator, and rearranging the result gives
  \[  A u \geq 0 ~~\text{with}~~
      A:u\in\R^d\mapsto \DP{u}{\relax
        \grad[\theta]{C(\theta_0)}+\E[z\sim\mu_z]{\grad[\theta]{h_z(\theta_0)}}}.
  \]
  Writing the same for unit vector $-u$ yields inequality $-Au\geq0$.
  Therefore $Au=0$.
\end{proof}

Thanks to this result, we can compute an unbiased\footnote{\relax
  Stochastic gradient descent often relies on unbiased gradient
  estimates (for a more general condition,
  see~\cite[Assumption~4.3]{bottou-curtis-nocedal-2016}). This is not
  a given: estimating the Wasserstein distance~\eqref{eq:emdual} and
  its gradients on small minibatches gives severely biased
  estimates~\cite{bellemare-2017}. This is in fact very obvious for
  minibatches of size one. Theorem~\ref{th:envelope} therefore
  provides an imperfect but useful alternative.}
stochastic estimate $\hat{g}(\theta_t)$ of the gradient
$\grad_{\theta}C(\theta_t)$ by first solving the maximization problem
in~\eqref{eq:problem}, and then using the back-propagation algorithm
to compute the average gradient on a minibatch~$z_1\dots{z_k}$ sampled
from $\mu$,
\[
  \hat{g}(\theta_t)
    = - \frac{1}{k} \sum_{i=1}^{k} \grad[\theta]{\fp^*(G_\theta(z_i))}~.
\]
Such an unbiased estimate can then be used to perform a stochastic
gradient descent update iteration on the generator parameter
\[
  \theta_{t+1} = \theta_t - \eta_t\, \hat{g}(\theta_t)~.
\]
Although this algorithmic idea can be made to work relatively
reliably~\cite{arjovsky-chintala-bottou-2017,gulrajani-2017},
serious conceptual and practical issues remain:

\begin{remark}
  \label{rem:arora}
  In order to obtain an unbiased gradient estimate
  $\hat{g}(\theta_t)$, we need to solve the maximization problem
  in~\eqref{eq:problem} for the true distributions rather than for a
  particular subset of examples.  On the one hand, we can use the
  standard machine learning toolbox to avoid overfitting the
  maximization problem. On the other hand, this toolbox essentially
  works by restricting the family $\calQ$ in ways that can change the
  meaning of the comparison criteria itself
  \cite{arora-2017,liu-bousquet-chaudhuri-2017}.
\end{remark}

\begin{remark}
  \label{rem:twotimescale}
  In practice, solving the maximization problem~\eqref{eq:problem}
  during each iteration of the stochastic gradient algorithm is
  computationally too costly. Instead, practical algorithms interleave
  two kinds of stochastic iterations: gradient ascent steps on
  $(\fq,\fp)$, and gradient descent steps on $\theta$, with a much
  smaller effective stepsize. Such algorithms belong to the general
  class of stochastic algorithms with two time scales
  \cite{borkar-1997,konda-tsitsiklis-2004}. Their convergence
  properties form a delicate topic, clearly beyond the purpose of this
  contribution.
\end{remark}

\subsection{Integral probability metrics}
\label{sec:ipm}

Integral probability metrics~(IPMs)~\cite{mueller-1997} have the form
\[
   D(Q,P) = \left|~\sup_{f\in\calQ} \E[Q]{f(X)} - \E[P]{f(X)}~\right|\,.
\]
Note that the surrounding absolute value can be eliminated by
requiring that $\calQ$ also contains the opposite of every one of its
functions.

\begin{equation}
  \label{eq:ipmform}
  \begin{split}
  D(Q,P) &= \sup_{f\in\calQ} \E[Q]{f(X)} - \E[P]{f(X)}\,\\
         &\text{where $\calQ$ satisfies}\quad
  \forall f\in\calQ\,,~-f\in\calQ~.
  \end{split}
\end{equation}

Therefore an IPM is a special case of \eqref{eq:vform} where the
critic functions~$\fq$ and~$\fp$ are constrained to be identical, and
where $\calQ$ is again constrained to contain the opposite of every
critic function. Whereas expression~\eqref{eq:vform} does not
guarantee that $D(Q,P)$ is finite and is a distance, an IPM is always
a pseudodistance.

\begin{proposition}
  \label{th:ipmproperties}
  Any integral probability metric $D$,~\eqref{eq:ipmform} is a pseudodistance.
\end{proposition}
\begin{proof}
  To establish the triangular inequality~(\ref{eq:metricdistance}.$iii$),
  we can write, for all $Q,P,R\in\calM$,
  \begin{eqnarray*}
    D(Q,P)+D(P,R)
    &=& \sup_{f_1,\,f_2\in\calQ} \E[Q]{f_1(X)} - \E[P]{f_1(X)} + \E[P]{f_2(X)} - \E[R]{f_2(X)} \\
    &\geq& \sup_{f_1=f_2\in\calQ} \E[Q]{f_1(X)} - \E[P]{f_1(X)} + \E[P]{f_2(X)} - \E[R]{f_2(X)} \\
    &=& ~~ \sup_{f\in\calQ} ~~ \E[Q]{f(X)} - \E[R]{f(X)} ~=~ D(Q,R)~.
  \end{eqnarray*}
  The other properties of a pseudodistance are trivial consequences
  of~\eqref{eq:ipmform}.
\end{proof}

The most fundamental IPM is the Total Variation (TV) distance.
\begin{equation}
  \label{eq:totalvariation}
  D_{TV}(Q,P)~\stackrel{\Delta}{=}
  ~\sup_{A\in\mathfrak{U}} |P(A)-Q(A)|~=
  \sup_{f\in C(\calX,[0,1])} \E[Q]{f(x)}- \E[P]{f(x)} ~,
\end{equation}
where $C(\calX,[0,1])$ is the space of continuous functions from $\calX$ to $[0,1]$.

\subsection{\eff-Divergences}
\label{sec:fdivergences}

Many classical criteria belong to the family of \eff-divergences
\begin{equation}
  \label{eq:fdivergence}
  D_{\eff\,}(Q,P) ~\stackrel{\Delta}{=}~
    \int \eff\left(\frac{q(x)}{p(x)}\right) p(x)\,d\mu(x)
\end{equation}
where $p$ and $q$ are respectively the densities of $P$ and $Q$
relative to measure $\mu$ and where $\eff\,$ is a continuous convex
function defined on $R_+^*$ such that $\eff\,(1)=0$.

Expression \eqref{eq:fdivergence} trivially
satisfies~(\ref{eq:metricdistance}.$\mzero$).
It is always nonnegative because we can pick a
subderivative $u\in\partial\eff\,(1)$ and use the inequality
$\eff\,(t)\,\geq\,u(t-1)$. This also shows that the separation
property~(\ref{eq:metricdistance}.$i$) is satisfied when
this inequality is strict for all $t\neq1$.

\smallskip
\goodbreak

\begin{proposition}[\cite{nguyen-2010,nowozin-cseke-tomioka-2016} (informal)]
  \label{th:fdivergence}
  Usually,\footnote{\relax The statement holds when there is an $M{>}0$
    such that $\mu\{x:|\eff\,(q(x)/p(x))|{>}M\}{=}0$ Restricting $\mu$ to
    exclude such subsets and taking the limit $M\rightarrow\infty$ may
    not work because $\lim\sup\neq\sup\lim$ in general.  Yet, in
    practice, the result can be verified by elementary calculus for
    the usual choices of $\eff$, such as those shown in
    Table~\ref{tbl:fdivergences}.}
  \[
     D_{\eff\,}(Q,P) ~~ =
     \sup_{\stacktwo{g\:\mathrm{bounded,\,measurable}}{g(\calX)\subset\mathrm{dom}(\eff^{\:*\!})}}
       \E[Q]{g(x)}- \E[P]{\eff^{\:*}(g(x))} ~.
  \]
  where $\eff^{\:*}\!$ denotes the convex conjugate of $\eff$.
\end{proposition}

Table~\ref{tbl:fdivergences} provides examples of \eff-divergences and
provides both the function~$\eff$ and the corresponding conjugate
function~$\eff^{\:*}$ that appears in the variational formulation.  In
particular, as argued in \cite{nowozin-cseke-tomioka-2016}, this
analysis clarifies the probability comparison criteria associated with
the early GAN variants~\cite{goodfellow-2014}.

\begin{table}
  \caption{\label{tbl:fdivergences}
    Various \eff-divergences and the corresponding $\eff$ and $\eff^{\:*}$.}
  \centering
  \begin{tabular}{l@{\quad}ccc}
    & $\mathbf{\eff\,(t)}$ & \textbf{dom}$\mathbf{(\eff^{\:*})}$ & $\mathbf{\eff^{\:*}(u)}$ \\
    \hline\noalign{\smallskip}\relax
    Total variation \eqref{eq:totalvariation}
    &  $\tfrac12|t-1|$ & $[-\tfrac12,\tfrac12]$ & $u$ \\
    Kullback-Leibler \eqref{eq:kldivergence}
    &  $t\log(t)$ & $\R$ & $\exp({u-1})$ \\
    Reverse Kullback-Leibler
    & $-\log(t)$ & $\R_-$ & $-1-\log(-u)$ \\
    GAN's Jensen Shannon \cite{goodfellow-2014}
    & $t\log(t)-(t+1)\log(t+1)$ & $\R_-$ & $-\log(1-\exp(u))$ \\
    \hline
  \end{tabular}
  \par\bigskip
\end{table}

Despite the elegance of this framework, these comparison criteria are
not very attractive when the distributions are supported by
low-dimensional manifolds that may not overlap.
The following simple example shows how this can be a
problem~\cite{arjovsky-chintala-bottou-2017}.

\begin{figure}[t]
  \centering
  \begin{tikzpicture}[scale=0.8]
    \draw[->] (0,0) -- (8,0) node[anchor=north, inner sep=5pt] {$\theta$};
    \draw[->] (0,0) -- (0,3) node[anchor=east,  inner sep=5pt] {$t$};
    \draw (0.000, 0.000) node[anchor=north, inner sep=5pt] {$P_0$};
    \draw (0.875, 0.000) node[anchor=north, inner sep=5pt] {$P_{1/4}$};
    \draw (1.750, 0.000) node[anchor=north, inner sep=5pt] {$P_{1/2}$};
    \draw (3.500, 0.000) node[anchor=north, inner sep=5pt] {$P_1$};
    \draw (7.000, 0.000) node[anchor=north, inner sep=5pt] {$P_2$};
    \node [left] at (0, 0.000) {$0$};
    \node [left] at (0, 2.500) {$1$};
    \draw[] (0.000, 0.000) -- (0.000,  -0.100);
    \draw[] (0.875, 0.000) -- (0.875,  -0.100);
    \draw[] (1.750, 0.000) -- (1.750,  -0.100);
    \draw[] (3.500, 0.000) -- (3.500,  -0.100);
    \draw[] (7.000, 0.000) -- (7.000,  -0.100);
    \draw[very thick, red]   (0.000, 0.000) -- (0.000,  2.500);
    \draw[very thick, blue] (0.875, 0.000) -- (0.875,  2.500);
    \draw[very thick, blue] (1.750, 0.000) -- (1.750,  2.500);
    \draw[very thick, blue] (3.500, 0.000) -- (3.500,  2.500);
    \draw[very thick, blue] (7.000, 0.000) -- (7.000,  2.500);
  \end{tikzpicture}
  \caption{\label{fig:segments} Let distribution $P_\theta$ be
    supported by the segment $\{(\theta,t)~t\in[0,1]\}$ in $\R^2$.
    According to both the TV distance \eqref{eq:totalvariation} and
    the \eff-divergences \eqref{eq:fdivergence}, the sequence of
    distributions $(P_{1/i})$ does not converge to $P_0$. However
    this sequence converges to $P_0$ according to either the
    Wasserstein distances \eqref{eq:wasserstein} or the Energy
    distance~\eqref{eq:energydistance}.}
\end{figure}
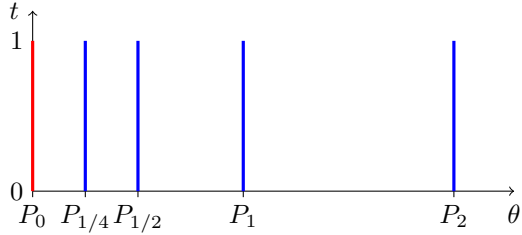

\begin{example}
  \label{ex:segments}
  Let $U$ be the uniform distribution on the real segment $[0,1]$ and
  consider the distributions $P_\theta=(\theta,x)\pfd U(x)$ defined on
  $\R^2$.  Because $P_0$ and $P_\theta$ have disjoint support for
  $\theta\neq0$, neither the total variation distance
  $D_{TV}(P_0,P_\theta)$ nor the \eff-divergence
  $D_{\eff\:}(P_0,P_\theta)$ depend on the exact value of
  $\theta$. Therefore, according to the topologies induced by these
  criteria on $\calM$, the sequence of distributions $(P_{1/i})$
  does not converge to $P_0$ (Figure~\ref{fig:segments}).
\end{example}

The fundamental problem here is that neither the total variation
distance~\eqref{eq:totalvariation} nor the
\eff-divergences~\eqref{eq:fdivergence} depend on the
distance~$d(x,y)$ defined on the sample space $\calX$. The
minimization of such a criterion appears more effective for adjusting
the probability values than for matching the distribution supports.

\subsection{Wasserstein distance}
\label{sec:wasserstein}

For any $p\geq1$, the $p$-Wasserstein distance (WD) is the $p$-th root of
\begin{equation}
  \label{eq:wasserstein}
  \forall Q,P\in\calM^p\qquad
  W_p(Q,P)^p ~\stackrel{\Delta}{=}~
   \inf_{\pi\in\Pi(Q,P)} \E[(x,y)\sim\pi]{\,d(x,y)^p\,}~,
\end{equation}
where $\Pi(Q,P)$ represents the set of all measures $\pi$ defined on
$\calX\times\calX$ with marginals $x\pfd\pi(x,y)$ and $y\pfd\pi(x,y)$ 
respectively equal to $Q$ and $P$. Intuitively, $d(x,y)^p$ represents
the cost of transporting a grain of probability from point~$x$ to
point~$y$, and the joint distributions~$\pi\in\Pi(Q,P)$ represent
transport plans.

Since $d(x,y)\leq d(x,x_0)+d(x_0,y)\leq2\max\{d(x,x_0),d(x_0,y)\}$,
\begin{equation}
  \label{eq:wpdomain}
  \forall\,Q,P\in\calM^p \qquad
    W_p(Q,P)^p \leq \E[\stacktwo{x\sim Q}{y\sim P}]{d(x,y)^p} < \infty~.
\end{equation}

\begin{example}
  Let $P_\theta$ be defined as in Example~\ref{ex:segments}.  Since it
  is easy to see that the optimal transport plan from $P_0$ to
  $P_\theta$ is~$\pi^*=((0,t),(\theta,t))\pfd U(t)$, the
  Wassertein distance $W_p(P_0,P_\theta)=|\theta|$ converges to zero
  when $\theta$ tends to zero. Therefore, according to the topology
  induced by the Wasserstein distance on $\calM$, the sequence of
  distributions~$(P_{1/i})$ converges to $P_0$
  (Figure~\ref{fig:segments}).
\end{example}

Thanks to the Kantorovich duality theory, the Wasserstein distance is
easily expressed in the variational form \eqref{eq:vform}.  We
summarize below the essential results useful for this work and we
direct the reader to \cite[Chapters~4~and~5]{villani-2009}
for a full exposition.

\begin{theorem}[{\cite[Theorem~4.1]{villani-2009}}]
  \label{th:existenz}
  Let $\calX,\calY$ be two Polish metric spaces and
  $c:\calX\times\calY\rightarrow\R_+\cup\{+\infty\}$ be a nonnegative
  continuous cost function. Let $\Pi(Q,P)$ be the set of probablity
  measures on $\calX\times\calY$ with marginals $Q\in\calM$ and
  $P\in\pspace{\calY}$. There is a $\pi^*\in\Pi(Q,P)$ that minimizes
  $\E[(x,y)\sim\pi]{c(x,y)}$ over all $\pi\in\Pi(Q,P)$.
\end{theorem}

\begin{definition}
  Let $\calX,\calY$ be two Polish metric spaces and
  $c:\calX{\times}\calY\rightarrow\R_+{\cup}\{+\infty\}$ be a
  nonnegative continuous cost function. The pair of functions
  $f:\calX\rightarrow\R$ and $g:\calY\rightarrow\R$ is $c$-conjugate
  when
  \begin{equation}
    \label{eq:cconjugate}
    \forall x\in\calX~ f(x)=\inf_{y\in\calY} g(y)+c(x,y) \quad\text{and}\quad
    \forall y\in\calY~ g(y)=\sup_{x\in\calX} f(x)-c(x,y)~.
  \end{equation}
\end{definition}


\medskip

\begin{theorem}[Kantorovich duality {\cite[Theorem~5.10]{villani-2009}}]
  \label{th:kantorovich}
  Let $\calX$ and $\calY$ be two Polish metric spaces and
  $c:\calX{\times}\calY\rightarrow\R_+\cup\{+\infty\}$ be a nonnegative
  continuous cost function. For all $Q\in\calM$ and
  $P\in\pspace{\calY}$, let $\Pi(Q,P)$ be the set of probability
  distributions defined on $\calX\times\calY$ with marginal
  distributions $Q$ and $P$. Let $\calQ_c$ be the set of all pairs
  $(\fq,\fp)$ of respectively $Q$ and $P$-integrable functions
  satisfying the property $\forall x\in\calX~y\in\calY$,
  $\fq(x)-\fp(y)\leq c(x,y)$.
  
  \begin{itemize}
  \item[$i$)]
    We have the duality
    \begin{align}
      \label{eq:kantoprimal}
      \min_{\pi\in\Pi(Q,P)} & \E[(x,y)\sim\pi]{ c(x,y)} ~~=~~  \\
      \label{eq:kantodual}
      & \sup_{(\fq,\fp)\in\calQ_c} \E[x\sim Q]{\fq(x)} - \E[y\sim P]{\fp(y)}~.
    \end{align}
  \item[$ii$)] Further assuming that $\E[{x\sim Q}\,{y\sim P}]{c(x,y)}<\infty$,\par
    \begin{itemize}[nosep]
    \item[a)] Both \eqref{eq:kantoprimal} and \eqref{eq:kantodual}
      have solutions with finite cost.
    \item[b)] The solution $(\fq^*,fp^*)$ of~\eqref{eq:kantodual} is a $c$-conjugate pair.
    \end{itemize}
  \end{itemize}
\end{theorem}

\medskip

\begin{corollary}[{\cite[Particular~case~5.16]{villani-2009}}]
  \label{th:kantomore}
  Under the same conditions as Theorem~\ref{th:kantorovich}.$ii$, when
  $\calX=\calY$ and when the cost function $c$ is a distance, that is,
  satisfies \eqref{eq:metricdistance}, the dual optimization problem
  \eqref{eq:kantodual} can be rewritten as
  \[
      \max_{f\in\mathrm{Lip1}} \E[Q]{f(x)}-\E[P]{f(x)}~,
  \]
  where $\mathrm{Lip1}$ is the set of real-valued $1$-Lipschitz
  continuous functions on $\calX$.
\end{corollary}

\smallskip
Thanks to Theorem~\ref{th:kantorovich}, we can write the $p$-th power
of the $p$-Wasserstein distance in variational form
\begin{equation}
  \label{eq:wpdual}
    \forall Q,P\in\calM^p \qquad
    W_p(Q,P)^p = \sup_{(\fq,\fp)\in\calQ_c} \E[Q]{\fq(x)}-\E[P]{\fp(x)}~,
\end{equation}
where $\calQ_c$ is defined as in Theorem~\ref{th:kantorovich} for the
cost $c(x,y)=d(x,y)^p$. Thanks to Corollary~\ref{th:kantomore}, we can
also obtain a simplified expression in IPM form for the
$1$-Wasserstein distance.
\begin{equation}
  \label{eq:emdual}
  \forall Q,P\in\calM^1 \qquad
  W_1(Q,P) = \sup_{f\in\mathrm{Lip1}} \E[Q]{f(x)}-\E[P]{f(x)}~.
\end{equation}

Let us conclude this presentation of the Wassertein distance by
mentioning that the definition~\eqref{eq:wasserstein} immediately
implies several distance properties: zero when both
distributions are equal~(\ref{eq:metricdistance}.$\mzero$), strictly
positive when they are different~(\ref{eq:metricdistance}.$i$), and
symmetric~(\ref{eq:metricdistance}.$ii$). Property~\ref{th:ipmproperties}
gives the triangular inequality (\ref{eq:metricdistance}.$iii$) for
the case $p=1$.  In the general case, the triangular inequality can
also be established using the Minkowsky inequality
\cite[Chapter~6]{villani-2009}.

\subsection{Energy Distance and Maximum Mean Discrepancy}
\label{sec:ed-mmd}

The Energy Distance (ED) \cite{szekely-2002} between the probability
distributions $Q$ and $P$ defined on the Euclidean space $\R^d$ is
the square root\footnotemark~of \footnotetext{We take the square root
  because this is the quantity that behaves like a distance.}
\begin{equation}
  \label{eq:energydistance}
  \calE(Q,P)^2 ~\stackrel{\Delta}{=}~
  2\E[\stacktwo{x\sim{Q}}{y\sim{P}}]{\|x-y\|}
  - \E[\stacktwo{x\sim{Q}}{x'\!\sim{Q}}]{\|x-x'\|}
  - \E[\stacktwo{y\sim{P}}{y'\!\sim{P}}]{\|y-y'\|}~,
\end{equation}
where, as usual, $\|\cdot\|$ denotes the Euclidean distance.

Let $\hat{q}$ and $\hat{p}$ represent the characteristic functions
of the distribution $Q$ and $P$ respectively. Thanks to a neat
Fourier transform argument \cite{szekely-2002,szekely-rizzo-2013},
\begin{equation}
  \label{eq:szekely}
  \calE(Q,P)^2 = \frac{1}{c_d} \int_{\R^d}
  \frac{|\hat{q}(t)-\hat{p}(t)|^2}{\|t\|^{d+1}} dt \quad
  \text{with~ $c_d=\frac{\pi^{\tfrac{d+1}{2}}}{\Gamma(\tfrac{d+1}{2})}$}~.
\end{equation}
Since there is a one-to-one mapping between distributions and
characteristic functions, this relation establishes an isomorphism
between the space of probability distributions equipped with the ED
distance and the space of the characteristic functions equipped with
the weighted $L_2$ norm given in the right-hand side
of~\eqref{eq:szekely}. As a consequence, $\calE(Q,P)$ satisfies
the properties \eqref{eq:metricdistance} of a distance.

Since the squared ED is expressed with a simple combination of
expectations, it is easy to design a stochastic minimization algorithm
that relies only on two oracles producing samples from each
distribution~\cite{bouchacourt-2016,bellemare-2017}.  This makes the
energy distance a computationally attractive criterion for training
the implicit models discussed in Section~\ref{sec:implicitmodeling}.

\medskip
\paragraph{Generalized ED~}
It is therefore natural to ask whether we can meaningfully
generalize~\eqref{eq:energydistance} by replacing the Euclidean
distance $\|x-y\|$ with a symmetric function $d(x,y)$.
\begin{equation}
  \label{eq:genergydistance}
  \calE_d(Q,P)^2 ~=~
  2\E[\stacktwo{x\sim{Q}}{y\sim{P}}]{d(x,y)}
  - \E[\stacktwo{x\sim{Q}}{x'\!\sim{Q}}]{d(x,x')}
  - \E[\stacktwo{y\sim{P}}{y'\!\sim{P}}]{d(y,y')}~.
\end{equation}
The right-hand side of this expression is well defined when
$Q,P\in\calM^1$. It is obviously
symmetric~(\ref{eq:metricdistance}.$ii$) and trivially
zero~(\ref{eq:metricdistance}.$\mzero$) when both distributions are
equal. The first part of the following theorem gives the necessary and
sufficient conditions on $d(x,y)$ to ensure that the right-hand side
of \eqref{eq:genergydistance} is nonnegative and therefore can be the
square of $\calE_d(Q,P)\in\R_+$.  We shall see later that the
triangular inequality~(\ref{eq:metricdistance}.$iii$) comes for free
with this condition (Corollary~\ref{th:eddistance}). The second part
of the theorem gives the necessary and sufficient condition for
satisfying the separation property~(\ref{eq:metricdistance}.$i$).

\goodbreak

\begin{theorem}[\cite{zinger-1992}]
  \label{th:zinger}
  The right-hand side of definition~\eqref{eq:genergydistance} is:
  \begin{itemize}
  \item[$i$)] nonnegative for all $P,Q$ in $\calM^1$ if and only
    if the symmetric function $d$ is a
    \emph{negative~definite~kernel}, that is, 
    \begin{multline}
      \label{eq:ndkernel}
      \forall n\in\mathbb{N} ~~
      \forall x_1\dots x_n\in\calX ~~
      \forall c_1\dots c_n\in\R \quad \\
      \sum_{i=1}^n c_i=0 ~~ \Longrightarrow
      ~ \sum_{i=1}^n \sum_{j=1}^n d(x_i,x_j) c_i c_j  \leq  0~.
    \end{multline}
    
  \item[$ii$)] strictly positive for all $P\neq Q$ in $\calM^1$ if and
    only if the function~$d$ is a
    \emph{strongly~negative~definite~kernel}, that is, a negative
    definite kernel such that, for any probability measure
    $\mu\in\calM^1$ and any $\mu$-integrable real-valued function $h$
    such that $\E[\mu]{h(x)}=0$,
    \[
       \E[\stacktwo{x\sim\mu}{y\sim\mu}]{\,d(x,y)h(x)h(y)\,}=0
         ~~\Longrightarrow~~  h(x)=0 ~~\text{$\mu$-almost everywhere.}
    \]
  \end{itemize}
\end{theorem}

\begin{remark}
  The definition of a strongly negative kernel is best explained by
  considering how its meaning would change if we were only considering
  probability measures $\mu$ with finite support $\{x_1\dots x_n\}$.
  This amounts to requiring that \eqref{eq:ndkernel} is an equality
  only if all the $c_i$s are zero. However, this weaker property is
  not sufficient to ensure that the separation
  property~(\ref{eq:metricdistance}.$i$) holds.
\end{remark}

\begin{remark}
  \label{rem:euclidianstronglynegative}
  The relation~\eqref{eq:szekely} therefore means that the Euclidean
  distance on $\R^d$ is a strongly negative definite kernel. In fact,
  it can be shown that $d(x,y)=\|x-y\|^\beta$ is a strongly negative
  definite kernel for $0<\beta<2$ \cite{szekely-rizzo-2013}. When $\beta=2$,
  it is easy to see that $\calE_d(Q,P)$ is simply the distance
  between the distribution means and therefore cannot
  satisfy the separation property (\ref{eq:metricdistance}.$i$).
\end{remark}

\begin{proofx}{Proof of Theorem~\ref{th:zinger}}
  Let $E(Q,P)$ be the right-hand side of \eqref{eq:genergydistance}
  and let $S(\mu,h)$ be the quantity $\E[x,y\sim \mu]{\,d(x,y)h(x)h(y)\,}$
  that appears in clause ($ii$). Observe:
  \smallskip
  \begin{itemize}[nosep]
  \item[a)] Let $Q,P\in\calM^1$ have respective density functions $q(x)$
    and $p(x)$ with respect to measure $\mu=(Q+P)/2$.
    Function $h=q-p$ then satisfies $\E[\mu]{h}{=}0$, and
    \[
    E(Q,P)
    =\E[\stacktwo{x\sim\mu}{y\sim\mu}]{\,
        \big(q(x)p(y)+q(y)p(x)-q(x)q(y)-p(x)p(y)\big)\,d(x,y)\,}
    = -S(\mu,h)~. 
    \]
  \item[b)] With $\mu\in\calM^1$, any $h$ such that $\mu\{h{=}0\}<1$
    (\ie., non-$\mu$-almost-surely-zero) and
    $\E[\mu]{h}{=}0$ can be written as a difference of two
    nonnegative functions $h{=}\tilde{q}-\tilde{p}$ such that
    $\E[\mu]{\tilde{q}}{=}\E[\mu]{\tilde{p}}{=}\rho^{-1}>0$.  Then,
    $Q=\rho\,\tilde{q}\,\mu$ and $P=\rho\,\tilde{p}\,\mu$
    belong to $\calM^1$, and
    \[ E(Q,P)=-\rho\,S(\mu,h)~. \]
  \end{itemize}
  We can then prove the theorem:
  \begin{itemize}[nosep]
  \item[$i$)]From these observations, if $E(Q,P)\geq0$ for all $P,Q$,
    then $S(\mu,h)\leq0$ for all $\mu$ and $h$ such that
    $\E[\mu]{h(x)}=0$, implying~\eqref{eq:ndkernel}. Conversely,
    assume there are $Q,P\in\calM^1$ such that $E(Q,P)<0$. Using
    the weak law of large numbers~\cite{khinchin-1929} (see also
    Theorem~\ref{th:edvstat} later in this document,) we can find
    finite support distributions $Q_n,P_n$ such that
    $E(Q_n,P_n)<0$. Proceeding as in observation~(a)
    then contradicts~\eqref{eq:ndkernel} because $\mu=(Q_n+P_n)/2$
    has also finite support.
  \item[$ii$)] By contraposition, suppose there is $\mu$ and $h$ such
    that $\mu\{h{=}0\}{<}1$, $\E[\mu]{h(x)}=0$, and
    $S(\mu,h)=0$. Observation (b) gives $P\neq Q$ such that
    $E(Q,P)=0$. Conversely, suppose $E(Q,P)=0$. Observation (a) gives
    $\mu$ and $h=q-p$ such that $S(\mu,h)=0$. Since $h$ must be zero,
    $Q=P$. \qedsymbol
  \end{itemize}
\end{proofx}

Requiring that $d$ be a negative definite kernel is a quite strong
assumption. For instance, a classical result by
Schoenberg~\cite{schoenberg-1938} establishes that a squared distance
is a negative definite kernel if and only if the whole metric space
induced by this distance is isometric to a subset of a Hilbert space
and therefore has a Euclidean geometry:

\begin{theorem}[{Schoenberg, \cite{schoenberg-1938}}]
  The metric space $(\calX,d)$ is isometric to a subset of a Hilbert space
  if and only if $d^2$ is a negative definite kernel.
\end{theorem}

Requiring $d$ to be negative definite (not
necessarily a squared distance anymore) has a similar impact on the
geometry of the space $\calM^1$ equipped with the Energy Distance
(Theorem~\ref{th:edmmd}).
Let $x_0$ be an arbitrary origin point and define the
symmetric \emph{triangular gap} kernel $K_d$ as
\begin{equation}
  \label{eq:triangulargap}
   K_d(x,y) ~\stackrel{\Delta}{=}~
    \tfrac12\left( d(x,x_0)+d(y,x_0)-d(x,y) \right)~.
\end{equation}

\begin{proposition}
  \label{th:negdefposdef}
  The function $d$ is a negative definite kernel if and only if
  $K_d$ is a positive definite kernel, that is,
  \vspace*{-\abovedisplayskip}
  \[
  \forall n\in\mathbb{N} ~~
  \forall x_1\dots x_n\in\calX ~~
  \forall c_1\dots c_n\in\R ~~
  \sum_{i=1}^n \sum_{j=1}^n c_i c_j K_d(x_i,x_j) \geq 0~.
  \]
\end{proposition}
\begin{proof} The proposition directly results from the identity
   \[
     2 \sum_{i=1}^n \sum_{j=1}^n c_i c_j K_d(x_i,x_j)
     = - \sum_{i=0}^n \sum_{j=0}^n c_i c_j d(x_i,x_j)~,
   \]
  where $x_0$ is the chosen origin point and $c_0=-\sum_{i=1}^{n}c_i$.    
\end{proof}

Positive definite kernels in the machine learning literature have been
extensively studied in the context of the so-called
\emph{kernel~trick}~\cite{schoelkopf-smola-2002}.  In particular, it
is well known that the theory of the Reproducing Kernel Hilbert Spaces
(RKHS)~\cite{aronszajn-1950,aizerman-1964} establishes that there is a
unique Hilbert space $\calH$, called the RKHS, that contains all the
functions
\[
  \Phi_x:y\in\calX\mapsto{K_d(x,y)}
\]
and satisfies the \emph{reproducing property} 
\begin{equation}
  \label{eq:reproducing}
  \forall x\in\calX ~~
  \forall f\in\calH \quad
  \DP{f}{\Phi_x} = f(x)~.
\end{equation}
We can then relate $\calE_d(Q,P)$ to the RKHS norm.

\begin{theorem}[{\cite{sejdinovic-2013} \cite[Chapter 21]{rachev-2013}}]
  \label{th:edmmd}
  Let $d$ be a negative definite kernel and let $\calH$ be the RKHS
  associated with the corresponding positive definite triangular gap
  kernel~\eqref{eq:triangulargap}.  We have then
  \[
  \forall Q,P\in\calM^1 \quad
  \calE_d(Q,P) ~=~ \|\: \E[x\sim Q]{\Phi_x} - \E[y\sim P]{\Phi_y} \:\|\subH ~.
  \]
\end{theorem}
\begin{proof} We can write directly
  \begin{eqnarray*}
    \calE_d(Q,P)^2
    &=& \E[\stacktwo{x,x'\sim Q}{y,y'\sim P}]{
      d(x,y) + d(x',y') - d(x,x') - d(y,y') } \\
    &=& \E[\stacktwo{x,x'\sim Q}{y,y'\sim P}]{
      K_d(x,x') + K_d(y,y') - K_d(x,y) - K_d(x',y') } \\
    &=& \DP{\E[Q]{\Phi_x}}{\E[Q]{\Phi_x}}+\DP{\E[P]{\Phi_y}}{\E[P]{\Phi_y}}
       -2\DP{\E[Q]{\Phi_x}}{\E[P]{\Phi_y}} \\
    &=& \| \E[x\sim Q]{\Phi_x} - \E[y\sim P]{\Phi_y} \|\subH^2 ~,
  \end{eqnarray*}
  where the first equality results from \eqref{eq:triangulargap} and
  where the second equality results from the identities
  $\DP{\Phi_x}{\Phi_y}=K_d(x,y)$ and
  $\E[x,y]{\DP{\Phi_x}{\Phi_y}}=\DP{\E[x]{\Phi_x}}{\E[y]{\Phi_y}}$.
\end{proof}
\begin{remark}
  In the context of this theorem, the relation~\eqref{eq:szekely} is
  simply an analytic expression of the RKHS norm associated with the
  triangular gap kernel of the Euclidean distance.
\end{remark}
\begin{corollary}
  \label{th:eddistance}
  If $d$ is a negative definite kernel, then $\calE_d$ is a
  pseudodistance, that is, it satisfies all the
  properties~\eqref{eq:metricdistance} of a distance except maybe the
  separation property~(\ref{eq:metricdistance}.$i$).
\end{corollary}
\begin{corollary}
  \label{th:characteristic}
  The following three conditions are then equivalent:
  \begin{itemize}[nosep]
  \item[$i$)] $\calE_d$ satisfies all the
    properties~\eqref{eq:metricdistance} of a distance.
  \item[$ii$)] $d$ is a strongly negative definite kernel.
  \item[$iii$)] the map $P\in\calM^1\mapsto\E[P]{\Phi_x}\in\calH$ is injective
     (characteristic kernel~\cite{gretton-2012}.)
  \end{itemize}
\end{corollary}

\medskip
\paragraph{Maximum Mean Discrepancy~}
Following \cite{gretton-2012}, we can then write $\calE_d$ as an IPM:
\begin{eqnarray}
  \calE_d(Q,P) &~=~& \|\E[Q]{\Phi_x}-\E[P]{\Phi_x}\|\subH\nonumber\\
  &=& \sup_{\|f\|\subH\leq1} \DP{f}{\E[P]{\Phi_x}-\E[Q]{\Phi_x}}  \nonumber\\
  &=& \sup_{\|f\|\subH\leq1} \E[P]{\DP{f}{\Phi_x}}-\E[Q]{\DP{f}{\Phi_x}} \nonumber\\
  &=& \sup_{\|f\|\subH\leq1} \E[P]{f(x)}-\E[Q]{f(x)}~. \label{eq:mmd}
\end{eqnarray}
This last expression \eqref{eq:mmd} is also called the Maximum Mean
Discrepancy (MMD) associated with the positive definite kernel $K_d$
\cite{gretton-2012}. Conversely, for any positive definite kernel $K$,
the reader will easily prove that the symmetric function
\[
   d_K(x,y) = \|\Phi_x-\Phi_y\|\subH^2 = K(x,x)+K(y,y)-2K(x,y)~,
\]
is a negative definite kernel, that $d_{K_d}=d$, and that
\begin{equation}
  \label{eq:dkdkd}
  \|\:\E[Q]{\Phi_x}-\E[P]{\Phi_x}\:\|\subH^2 ~=~ \calE_{d_K}(Q,P)^2~.
\end{equation}
Therefore the ED and MMD formulations are essentially equivalent
\cite{sejdinovic-2013}. Note however that the negative definite
kernel~$d_K$ defined above may not satisfy the triangular inequality
(its square root does.)

\begin{remark}
  Because this equivalence was not immediately recognized, many
  important concepts have been rediscovered with subtle technical
  variations. For instance, the notion of characteristic
  kernel~\cite{gretton-2012} depends subtly on the chosen domain for
  the map $P\mapsto\E[P]{\Phi_x}$ that we want injective.
  Corollary~\ref{th:characteristic} gives a simple necessary and
  sufficient condition when this domain is $\calM^1$ (with respect to
  the distance~$d$). Choosing a different domain leads to
  complications~\cite{sriperumbudur-2011}.
\end{remark}


\section{Energy Distance \vs. $1$-Wasserstein Distance}
\label{sec:emvsdisco}

The dual formulation of the $1$-Wasserstein \eqref{eq:emdual} and the
MMD formulation of the Energy Distance \eqref{eq:mmd} only differ
by the use of a different family of critic functions: for all $Q,P\in\calM^1$,
\begin{align*}
  W_1(Q,P) &= \sup_{f\in\mathrm{Lip1}} \E[Q]{f(x)}-\E[P]{f(x)} ~,\\
  \calE_d(Q,P) &= \sup_{\|f\|\subH\leq1} \E[P]{f(x)}-\E[Q]{f(x)} ~.
\end{align*}
At first sight, requiring that the functions $f$ are $1$-Lipschitz or
are contained in the RKHS unit ball seem to be two slightly different
ways to enforce a smoothness constraint. Nevertheless, a closer
comparison reveals very important differences.

\subsection{Three quantitative properties}
\label{sec:quantprop}

Although both the WD \cite[Theorem~6.9]{villani-2009} and the ED/MMD
\cite[Theorem~3.2]{sriperumbudur-2016} metrize the weak convergence
topology, they may be quantitatively very different and therefore hard
to compare in practical situations.  The following upper bound
provides a clarification.

\begin{proposition}
  \label{th:edlessthanwd}
  Let $\calX$ be equipped with a distance $d$ that
  is also a negative definite kernel.
  Let the $1$-Wasserstein distance
  $W_1$ and the Energy Distance $\calE_d$ be defined as in
  \eqref{eq:wasserstein} and \eqref{eq:genergydistance}.
  \[
      \calE_d(Q,P)^2 \leq 2 W_1(Q,P)~.
  \]
\end{proposition}
This inequality is tight. It is indeed easy to see
that it becomes an equality when both $P$ and $Q$ are
Dirac distributions. 

The proof relies on an elementary geometrical lemma:
\begin{lemma}
  Let $A,B,C,D$ be four points in $\calX$ forming a quadrilateral.
  The perimeter length $d(A,B)+d(B,C)+d(C,D)+d(D,A)$ is longer than
  the diagonal lenghts $d(A,C)+d(B,D)$.
\end{lemma}
\begin{proofx}{Proof of the lemma}
  Summing the following triangular inequalities yields the result.
  \begin{align*}
    d(A,C) &\leq d(A,B)+d(B,C)  &
    d(A,C) &\leq d(C,D)+d(D,A) \\
    d(B,D) &\leq d(B,C)+d(C,D)  &
    d(B,D) &\leq d(D,A)+d(A,B) \makebox[0pt][l]{\quad\qquad\qedsymbol}
  \end{align*}
\end{proofx}

\begin{proofx}{Proof of proposition~\ref{th:edlessthanwd}}
  \def\xp{x^\prime}
  \def\yp{y^\prime}
  Let $(x,y)$ and $(\xp,\yp)$ be two independent samples
  of the optimal transport plan $\pi$ with marginals $Q$ and $P$.
  Since they are independent,
  \[
    2\,\E[\stacktwo{x\sim Q}{y\sim P}]{d(x,y)}
    = \E[\stacktwo{~~(x,y)\sim\pi}{(\xp,\yp)\sim\pi}]{d(x,\yp)+d(\xp,y)} ~.
  \]
  Applying the lemma and rearranging
  \begin{multline*}
    2 \E[\stacktwo{x\sim Q}{y\sim P}]{d(x,y)}
    \leq \E[\stacktwo{~~(x,y)\sim\pi}{(\xp,\yp)\sim\pi}]{ d(x,y)+d(y,\yp)+d(\yp,\xp)+d(\xp,x) } \\
    = W_1(Q,P)+\E[\stacktwo{y\sim P}{\yp\!\sim P}]{d(y,\yp)}
        + W_1(Q,P)+\E[\stacktwo{x\sim Q}{\xp\!\sim Q}]{d(x,\xp)}\,. 
  \end{multline*}
  Moving the remaining expectations to the left-hand side gives the result.
  \qedsymbol   
\end{proofx}

In contrast, the following results not only show that $\calE_d$ can be
very significantly smaller than the $1$-Wasserstein distance, but also
show that this happens in the particularly important situation where
one approximates a distribution with a finite sample.

\begin{theorem}
  \label{th:edvstat}
  Let $Q,P\in\calM^1$ be two probability distributions on $\calX$.
  Let $x_1\dots x_n$ be $n$ independent $Q$-distributed random
  variables, and let \mbox{$Q_n=\tfrac1n\sum_{i=1}^n \delta_{x_i}$} be
  the corresponding empirical probability distribution. Let~$\calE_d$
  be defined as in~\eqref{eq:genergydistance} with a kernel satisfying
  $d(x,x)=0$ for all $x$ in $\calX$. Then,
  \[  \E[x_1\dots x_n\sim Q]{\calE_d(Q_n,P)^2}
  = \calE_d(Q,P)^2 + \tfrac{1}{n}\,\E[x,x^\prime\sim Q]{d(x,x^\prime)}~, \]
  and
  \[ \E[x_1\dots x_n\sim Q]{\calE_d(Q_n,Q)^2}
  = \tfrac{1}{n}\,\E[x,x^\prime\sim Q]{d(x,x^\prime)}
  = \mathcal{O}(n^{-1}) ~. \]
\end{theorem}

Therefore the effect of replacing $Q$ by its empirical approximation
disappears quickly, like $\mathcal{O}(1/n)$, when $n$ grows. This
result is not very surprising when one notices that $\calE_d(Q_n,P)$
is a V-statistic \cite{vonmises-1947,serfling-1980}. However it gives
a precise equality with a particularly direct proof.

\begin{proof}
  \def\xp{x^\prime}
  Using the following equalities in the
  definition \eqref{eq:genergydistance} gives the first result.
  \begin{align*}
    \E[x_1\dots x_n\sim Q]{\:\E[\stacktwo{x\sim Q_n}{y\sim P\,~}]{d(x,y)}\:}
    &~=~ \E[x_1\dots x_n\sim Q]{\:\frac1n\sum_{i=1}^{n} \E[y\sim P]{d(x_i,y)}\:} \\
    &\hskip-4em ~=~ \frac1n\sum_{i} \E[\stacktwo{x\sim Q}{y\sim P}]{d(x,y)}
    ~=~ \E[\stacktwo{x\sim Q}{y\sim P}]{d(x,y)}~.\\
    \E[x_1\dots x_n\sim Q]{\:\E[\stacktwo{x\sim Q_n}{\xp\!\sim Q_n}]{d(x,\xp)}\:}
    &~=~ \E[x_1\dots x_n\sim Q]{\:\frac1{n^2}\sum_{i\neq j} d(x_i,x_j)\:} \\
    &\hskip-4em~=~ \frac1{n^2} \sum_{i\neq j} \E[\stacktwo{x\sim Q}{y\sim Q}]{d(x,y)}
    ~=~ \left(1-\frac1n\right) \E[\stacktwo{x\sim Q}{y\sim Q}]{d(x,y)}~.
  \end{align*}
  Taking $Q=P$ then gives the second result.
\end{proof}

Comparable results for the $1$-Wasserstein distance describe a
convergence speed that quickly becomes considerably slower with the
dimension $d>2$ of the sample space $\calX$
\cite{sriperumbudur-2012,dereich-2013,fournier-guillin-2015}.

\begin{theorem}[\cite{fournier-guillin-2015}]
  \label{th:fournier}
  Let $\calX$ be $\R^d$, $d>2$, equipped with the usual Euclidean distance.
  Let $Q\in\pspace{R^d}^2$ and let $Q_n$ be defined as in Theorem~\ref{th:edvstat}.
  Then,
  \[ \E[x_1\dots x_n\sim Q]{\:W_1(Q_n,Q)\:} = \mathcal{O}(n^{-1/d}) ~. \]
\end{theorem}

The following example, inspired by \cite{arora-2017},
illustrates this slow rate and its consequences.

\begin{example}
  \label{ex:arora}
  Let $Q$ be a uniform distribution supported by the unit sphere in
  $\R^d$ equipped with the Euclidean distance. Let $x_1\dots x_n$ be
  $n$ points sampled independently from this distribution and let
  $Q_n$ be the corresponding empirical distribution. Let $x$ be an
  additional point sampled from $Q$. It is well known\footnote{\relax
    The curious reader can pick an expression of
    $F_d(t)=P\{\|x-x_i\|<t\}$ in~\cite{hammersley-1950}, then derive
    an asymptotic bound for $P\{\min_i\|x-x_i\|<t\}=1-(1-F_d(t))^n$.}
  that $\min_i\|x-x_i\|$ remains arbitrarily close to $\sqrt{2}$, say,
  greater than $1.2$, with arbitrarily high probability when
  $d\gg\log(n)$.  Therefore,
  \[  W_1(Q_n,Q)\geq 1.2 \quad\text{when $n\ll\exp(d)$.} \]
  In contrast, observe
  \[ W_1(Q,\delta_0) = W_1(Q_n,\delta_0) = 1 ~. \]
  In other words, as long as $n\ll\exp(d)$, a Dirac distribution in
  zero is closer to the empirical distribution than the actual
  distribution \cite{arora-2017}.
\end{example}

\medskip

Theorem~\ref{th:edvstat} and Example~\ref{ex:arora} therefore show
that $\calE_d(Q_n,Q)$ can be much smaller than $W_1(Q_n,Q)$.  They
also reveal that the statistical properties of the $1$-Wasserstein
distance are very discouraging. Since the argument of
Example~\ref{ex:arora} naturally extends to the $p$-Wasserstein
distance for all $p\geq1$, the problem seems shared by all Wasserstein
distances.

\begin{remark}
  In the more realistic case where the $1$-Lipschitz critic is
  constrained to belong to a parametric family with sufficient
  regularity, the bound of theorem~\ref{th:fournier} can be improved
  to $\mathcal{O}(\sqrt{\log(n)/n})$ with a potentially large constant
  \cite{arora-2017}. On the other hand, constraining the critic too
  severely might prevent it from distinguishing distributions that differ
  in meaningful ways.
\end{remark}

\subsection{WD and ED/MMD in practice}
\label{sec:martin}

Why should we consider the Wasserstein Distance when the Energy
Distance and Maximum Mean Discrepancy offer better statistical
properties (Section~\ref{sec:quantprop}) and more direct learning
algorithms~\cite{dziu-2015,li-2015,bouchacourt-2016}\,?

The most impressive achievement associated with the implicit modeling
approach certainly is the generation of photo-realistic random images
that resemble the images provided as training
data~\cite{denton-2015,radford-2015,karras-2017}. In apparent
contradiction with the statistical results of the previous section,
and with a couple notable exceptions discussed later in this section,
the visual quality of the images generated using models trained by
directly minimizing the MMD \cite{dziu-2015} usually lags behind those
obtained with the WD
\cite{arjovsky-chintala-bottou-2017,gulrajani-2017,karras-2017} and
with the original Generative Adversarial Network
formulation\footnote{\relax Note that it is then important to use the
  $\log(D)$ trick succinctly discussed in the original GAN paper
  \cite{goodfellow-2014}.} \cite{radford-2015}.

Before discussing the two exceptions, it is worth recalling that the
visual quality of the generated images is a peculiar way to benchmark
generative models. This is an incomplete criterion because it does not
ensure that the model generates images that cover all the space
covered by the training data. This is an interesting criterion because
common statistical metrics, such as estimates of the negative
log-likelihood, are generally unable to indicate which models generate
the better-looking images~\cite{theis-2016}. This is a finicky
criterion because, despite efforts to quantify visual quality with
well-defined scores~\cite{salimans-2016}, the evaluation of
the image quality fundamentally remains a beauty contest.
Figure~\ref{fig:generatedlsun} nevertheless shows a clear difference.

\begin{figure}
  \centering
  \includegraphics[width=.40\linewidth]{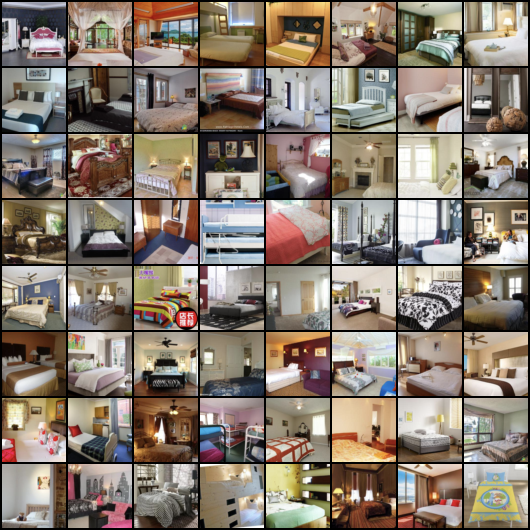}\\
  A sample of 64 training examples
  \par
  \bigskip
  \begin{tabular}{c@{\qquad}c}
  \includegraphics[width=.40\linewidth]{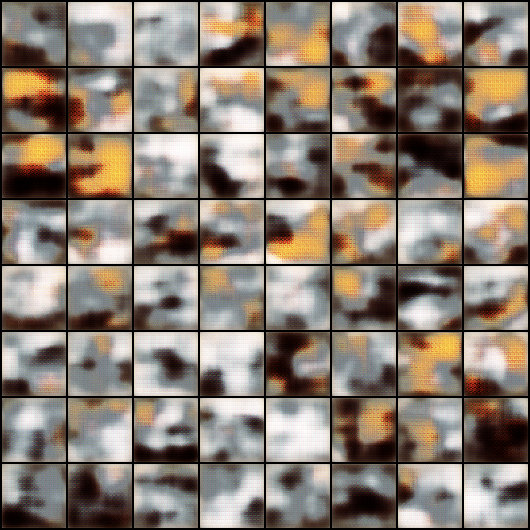} & 
  \includegraphics[width=.40\linewidth]{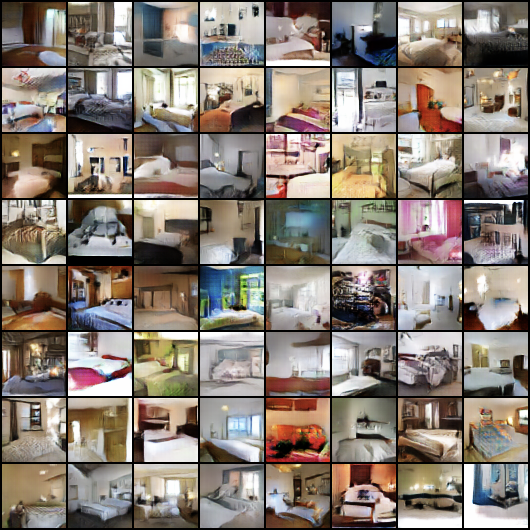} \\
  Generated by the ED trained model &
  Generated by the WD trained model
  \end{tabular}
  \caption{\label{fig:generatedlsun}
    Comparing images generated by a same implicit model trained with
    different criteria. The top square shows a sample of 64 training examples
    represening bedroom pictures. The bottom left square shows the images
    generated by a model trained with ED using the algorithm of \cite{bouchacourt-2016}.
    The bottom right square shows images generated by a model
    trained using the WGAN-GP approach \cite{gulrajani-2017}.}
\end{figure}

\smallskip
A few authors report good image generation results by using
the ED/MMD criterion in a manner that substantially changes
its properties:
\begin{itemize}[topsep=3pt,itemsep=4pt]
\item
  The AE+GMMN approach \cite{li-2015} improves the pure MMD approach by
  training an implicit model that does not directly generate images
  but targets the compact representation computed by a pretrained
  auto-encoder network. This changes a high-dimensional image
  generation problem into a comparatively low-dimensional code
  generation problem with a good notion of distance. There is
  independent evidence that low-dimensional implicit models work
  relatively well with ED/MMD \cite{bouchacourt-2016}.
\item
  The Cram\'{e}rGAN approach \cite{bellemare-2017} minimizes the
  Energy Distance\footnote{See \cite{szekely-2002} for the relation
    between Energy Distance and Cram\'{e}r distance.} computed on the
  representations produced by an adversarially trained $1$-Lipschitz
  continuous \emph{transformation layer} $T_\phi(x)$.  The resulting
  optimization problem
  \[
    \min_\theta ~ \left\{ \max_{T_\phi\in\mathrm{Lip1}}
      ~ \calE(T_\phi\pfd Q, T_\phi\pfd P_\theta ) \right\}\,,
  \]
  can then be re-expressed using the IPM form of the energy distance
  \[
  \min_\theta ~ \left\{
      D(Q,P) ~= \max_{T_\phi\in\mathrm{Lip1}} ~ \max_{\|f\|\subH\leq1} ~
        \E[x\sim Q]{f(T_\phi(x))}-\E[x\sim P_\theta]{f(T_\phi(x))} \right\} ~.
  \]
  The cost $D(Q,P)$ above is a new IPM that relies on critic functions
  of the form $f\circ{T_\phi}$, where~$f$ belongs to the RKHS unit
  ball, and~$T_\phi$ is $1$-Lipschitz continuous. Such hybrid critic
  functions still have smoothness properties comparable to that of the
  Lipschitz-continuous critics of the $1$-Wasserstein
  distance. However, since these critic functions do not usually form
  a RKHS ball, the resulting IPM criterion no longer belongs to the
  ED/MMD family.
\item
  The same hybrid approach gives comparable results in GMMN-C
  \cite{li-2017} where the authors replace autoencoder of GMMN+AE with
  an adversarially trained transformer layer.
\end{itemize}

On the positive side, such hybrid approaches may lead to more
efficient training algorithms than those described in
Section~\ref{sec:adversarialtraining}. The precise parametric
structure of the transformation layer also provides the means to match
what WGAN models achieve by selecting a precise parametric structure
for the critic. Yet, in order to understand these subtle effects, it
remains useful to clarify the similarities and differences between
pure ED/MMD training and pure WD training.

\section{Length spaces}
\label{sec:lengthspaces}

This section gives a concise review of the elementary metric geometry
concepts useful for the rest of our analysis. Readers can safely skip
this section if they are already familiar with metric geometry
textbooks such as \cite{burago-2001}.

\medskip
\paragraph{Rectifiable curves~}

A continuous mapping
$\gamma:t\in[a,b]\subset\R\mapsto\gamma_t\in\calX$ defines a curve
connecting $\gamma_a$ and $\gamma_b$. A curve is said to be
\emph{rectifiable} when its \emph{length}
\begin{equation}
  \label{eq:length}
  L(\gamma,a,b) \stackrel{\Delta}{~=~}
   \sup_{n>1} ~ \sup_{a=t_0<t_1<\dots<t_n=b} ~ \sum_{i=1}^{n} d(\gamma_{t_{i-1}},\gamma_t)
\end{equation}
is finite. Intuitively, thanks to the triangular inequality, dividing
the curve into $n$ segments $[\gamma_{t-1},\gamma_t]$ and summing
their sizes yields a quantity that is greater than
$d(\gamma_a,\gamma_b)$ but smaller than the curvilinear length of the
curve. By construction, $L(\gamma,a,b)\geq d(\gamma_a,\gamma_b)$ and
$L(\gamma,a,c)=L(\gamma,a,b)+L(\gamma,b,c)$ for all $a\leq b\leq c$.

\medskip
\paragraph{Constant speed curves~}

Together with the continuity of $\gamma$, this additivity property
implies that the function $t\in[a,b]\mapsto L(\gamma,a,t)$ is
nondecreasing and continuous \cite[Prop.~2.3.4]{burago-2001}.  Thanks
to the intermediate value theorem, when a curve is rectifiable, for
all $s\in[0,1]$, there is $t_s\in[a,b]$ such that
$L(\gamma,a,t_s)=s\,L(\gamma,a,b)$. Therefore, we can construct a new
curve $\gammab:s\in[0,1]\mapsto\gammab_s=\gamma_{t_s}$ that visits the
same points in the same order as curve $\gamma$ and satisfies the
property $\forall{s\in[0,1]}$, $L(\gammab,0,s)=s L(\gammab,0,1)$. Such
a curve is called a \emph{constant speed curve}.

\medskip
\paragraph{Length spaces~}

It is easy to check that the \emph{distance induced by $d$},
\begin{equation}
  \label{eq:induceddistance}
  \hat{d}:(x,y)\in\calX^2~~\mapsto\hskip-1em
  \inf_{\stacktwo{\gamma:[a,b]\rightarrow\calX}{\text{s.t. $\gamma_a=x$ $\gamma_b=y$}}}
    \hskip-1em L(\gamma,a,b) ~ \in \R_+^*\cup\{\infty\}\,,
\end{equation}
indeed satisfies all the properties \eqref{eq:metricdistance} of a
distance. It is also easy to check that the distance induced by
$\hat{d}$ coincides with $\hat{d}$ \cite[Prop.~2.3.12]{burago-2001}.
For this reason, a distance that satisfies $\hat{d}=d$ is called an
\emph{intrinsic distance}. A Polish metric space equipped with an
intrinsic distance is called an \emph{intrinsic Polish space}.
A metric space $\calX$ equipped with an intrinsic distance $d$ is
called a \emph{length space}.

\medskip
\paragraph{Minimal geodesics~}

A curve $\gamma:[a,b]\rightarrow\calX$ that achieves the infimum in
\eqref{eq:induceddistance} is called a \emph{shortest path} or a
\emph{minimal geodesic} connecting $\gamma_a$ and $\gamma_b$.

When the distance $d$ is intrinsic, the length of a minimal geodesic
$\gamma$ satisfies the relation
$L(\gamma,a,b)=\hat{d}(\gamma_a,\gamma_b)=d(\gamma_a,\gamma_b)$.  When
such a curve exists between any two points $x,y$ such that
$d(x,y)<\infty$, the distance $d$ is called \emph{strictly}
\emph{intrinsic}. A Polish space equipped with a strictly intrinsic
distance is called a \emph{strictly intrinsic Polish space.}

Conversely, a rectifiable curve $\gamma:[a,b]\rightarrow\calX$ of
length $d(\gamma_a,\gamma_b)$ is a minimal geodesic because no curve
joining $\gamma_a$ and $\gamma_b$ can be shorter. If there is such a
curve between any two points $x,y$ such that $d(x,y)<\infty$, then $d$
is a strictly intrinsic distance.

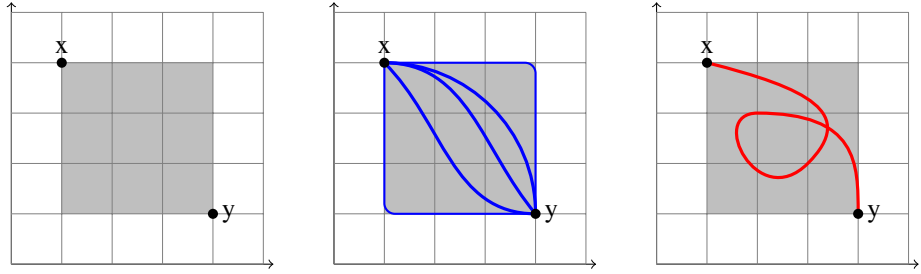
\begin{figure}
  \bigskip
  \centering
    \begin{tikzpicture}[scale=0.67]
      \draw [->] (1,1) -- (1,6.2);
      \draw [->] (1,1) -- (6.2,1);
      \path [fill=lightgray] (2,5) rectangle (5,2);
      \draw [help lines] (1,1) grid (6,6); 
      \path [fill=black] (2,5) circle [radius=0.1];
      \node [above] at (2,5) {x};
      \path [fill=black] (5,2) circle [radius=0.1];
      \node [right] at (5,2) {y};
    \end{tikzpicture}
    \qquad
    \begin{tikzpicture}[scale=0.67]
      \draw [->] (1,1) -- (1,6.2);
      \draw [->] (1,1) -- (6.2,1);
      \path [fill=lightgray] (2,5) rectangle (5,2);
      \draw [help lines] (1,1) grid (6,6); 
      \draw [blue,rounded corners, thick] (2,5) -- (2,2) -- (5,2);
      \draw [blue,rounded corners, thick] (2,5) -- (5,5) -- (5,2);
      \draw [blue,very thick] (2,5) to [out=0,in=130] (5,2);
      \draw [blue,very thick] (2,5) to [out=0,in=90] (5,2);
      \draw [blue,very thick] (2,5) to [out=-45,in=180] (5,2);
      \path [fill=black] (2,5) circle [radius=0.1];
      \node [above] at (2,5) {x};
      \path [fill=black] (5,2) circle [radius=0.1];
      \node [right] at (5,2) {y};
    \end{tikzpicture}
    \qquad
    \begin{tikzpicture}[scale=0.67]
      \draw [->] (1,1) -- (1,6.2);
      \draw [->] (1,1) -- (6.2,1);
      \path [fill=lightgray] (2,5) rectangle (5,2);
      \draw [help lines] (1,1) grid (6,6); 
      \draw [red,very thick] (2,5) 
      .. controls (4,4.5) and (5,4) .. (4,3)
      .. controls (3,2) and (2,4) .. (3,4)
      .. controls (5,4) and (5,3) .. (5,2);      
      \path [fill=black] (2,5) circle [radius=0.1];
      \node [above] at (2,5) {x};
      \path [fill=black] (5,2) circle [radius=0.1];
      \node [right] at (5,2) {y};
    \end{tikzpicture}
  \caption{\label{fig:l1geodesics}
    Consider $\R^2$ equipped with the $L_1$ distance. ~{Left}: all points $z$ in
    the gray area are such that $d(x,z)+d(z,y)=d(x,y)$. ~{Center}: all minimal
    geodesics connecting $x$ and $y$ live in the gray area. ~{Right}: but not all
    curves that live in the gray area are minimal geodesics.}
\end{figure}

\medskip
\paragraph{Characterizing minimal geodesics~}

Let $\gamma:[a,b]\rightarrow\calX$ be a minimal geodesic in a length
space $(\calX,d)$.  Using the triangular inequality and~\eqref{eq:length},
\begin{equation}
  \label{eq:triangularequality}
   \forall a\leq t\leq b \quad
     d(\gamma_a,\gamma_b)
      \leq d(\gamma_a,\gamma_t)+d(\gamma_t,\gamma_b)
       \leq L(\gamma,a,b)=d(\gamma_a,\gamma_b)~.
\end{equation}
This makes clear that every minimal geodesic in a length space is made
of points~$\gamma_t$ for which the triangular inequality is an
equality. However, as shown in Figure~\ref{fig:l1geodesics},
this is not sufficient to ensure that a curve is a minimal geodesic.
One has to consider two intermediate points:

\begin{theorem}
  \label{th:twopoints}
  Let $\gamma:[a,b]\rightarrow\calX$ be a curve joining
  two points $\gamma_a,\gamma_b$ such that $d(\gamma_a,\gamma_b)<\infty$.
  This curve is a minimal geodesic of length~$d(\gamma_a,\gamma_b)$
  if and only if~
  \( \forall ~ a\leq t \leq t' \leq b, \quad
       d(\gamma_a,\gamma_t)+d(\gamma_t,\gamma_{t'})+d(\gamma_{t'},\gamma_b)
        = d(\gamma_a,\gamma_b)\,. \)
\end{theorem}
\begin{corollary}
  \label{th:csg}
  Let $\gamma:[0,1]\rightarrow\calX$ be a curve joining two points
  $\gamma_0,\gamma_1\in\calX$ such that $d(\gamma_0,\gamma_1)<\infty$.
  The following three assertions are equivalent:
  \begin{itemize}[nosep]
  \item[a)] The curve~$\gamma$ is a constant speed minimal geodesic
    of length $d(\gamma_0,\gamma_1)$.
  \item[b)] \( \forall ~ t,t'\in[0,1], \quad
    d(\gamma_t,\gamma_{t'}) = |t-t'|\,d(\gamma_0,\gamma_1)\,. \)
  \item[c)] \( \forall ~ t,t'\in[0,1], \quad
    d(\gamma_t,\gamma_{t'}) \leq |t-t'|\,d(\gamma_0,\gamma_1)\,. \)
  \end{itemize}
\end{corollary}

\begin{proof}
  The necessity ($\Rightarrow$) is easily proven by rewriting 
  \eqref{eq:triangularequality} with two points $t$ and $t'$ instead
  of just one.  The sufficiency ($\Leftarrow$) is proven by induction. Let
  \[
    h_n = \sup_{a=t_0\leq t_1 \leq \dots \leq t_n \leq b} ~
    \sum_{i=1}^{n} d(\gamma_{t_{i-1}},\gamma_{t_i})~ + ~ d(\gamma_{t_n},\gamma_b)~.
  \]
  The hypothesis implies that $h_2=d(\gamma_a,\gamma_b)$.
  We now assuming that the induction hypothesis $h_n=d(\gamma_a,\gamma_b)$ is true for some $n\geq2$.
  For all partition $a=t_0\leq t_1\dots t_n\leq b$,
  using twice the triangular inequality and the induction hypothesis,
  \[
    d(\gamma_a,\gamma_b) \leq d(\gamma_a,\gamma_{t_n})+d(\gamma_{t_n},\gamma_b)
    \leq \sum_{i=1}^{n} d(\gamma_{t_{i-1}},\gamma_{t_i}) + d(\gamma_{t_n},\gamma_b) 
    \leq h_n = d(\gamma_a,\gamma_b)\,.
  \]
  Therefore $\sum_{i=1}^{n} d(\gamma_{t_{i-1}},\gamma_{t_i})=d(\gamma_a,\gamma_{t_n})$.
  Then, for any $t_{n+1}\in[t_n,b]$,
  \[
    \sum_{i=1}^{n+1} d(\gamma_{t_{i-1}},\gamma_{t_i}) + d(\gamma_{t_{n+1}},\gamma_b) =
    d(\gamma_a,\gamma_{t_n}) + d(\gamma_{t_n},\gamma_{t_{n+1}}) + d(\gamma_{t_{n+1}},\gamma_b) 
    = d(\gamma_a,\gamma_b)\,.
  \]
  Since this is true for all partitions, $h_{n+1}=d(\gamma_a,\gamma_b)$.
  We just have proved by induction that $h_n=d(\gamma_a,\gamma_b)$ for all $n$.
  Therefore $L(\gamma,a,b)=\sup_n h_n = d(\gamma_a,\gamma_b)$.
\end{proof}

\section{Minimal geodesics in probability space}
\label{sec:geodesics}

We now assume that $\calX$ is a strictly intrinsic Polish space
and we also assume that its distance $d$ is never infinite. Therefore
any pair of points in $\calX$ is connected by at least one minimal
geodesic. When the space~$\calM$ of probability distributions is
equipped with one of the probability distances discussed in
section~\ref{sec:adversarialnets}, it often becomes a length space
itself and inherits some of the geometrical properties of
$\calX$. Since this process depends critically on how the probability
distance compares different distributions, understanding the geodesic
structure of $\calM$ reveals fundamental differences between
probability distances.

This approach is in fact quite different from the celebrated work of
Amari on \emph{Information Geometry} \cite{amari-2007}. We seek here
to understand the geometry of the space of all probability measures
equipped with different distances. Information Geometry characterizes
the Riemannian geometry of a parametric family of probability measures
under the Kullback-Leibler distance. This difference is obviously
related to the contrast between \emph{relying on good distances}
versus \emph{relying on good model families} discussed in
Section~\ref{sec:comparingprobas}. Since we are particularly
interested in relatively simple models that have a physical or causal
interpretation but cannot truly represent the actual data
distribution, we cannot restrict our geometrical insights to what
happens within the model family.

\subsection{Mixture geodesics}
\label{sec:mixturegeodesics}

For any two distributions $P_0,P_1\in\calM$, the mixture distributions
\begin{equation}
  \label{eq:mixture}
  \forall t\in[0,1] \quad P_t = (1{-}t) P_0 + t P_1
\end{equation}
form a curve in the space of distributions $\calM$.

\begin{theorem}
  \label{th:mixtureipm}
  Let $\calM$ be equipped with a distance $D$ that belongs to the IPM
  family~\eqref{eq:ipmform}. Any mixture curve \eqref{eq:mixture}
  joining two distributions $P_0,P_1\in\calM$ such that
  $D(P_0,P_1)<\infty$ is a constant speed minimal geodesic,
  making $D$ a strictly intrinsic distance.
\end{theorem}

\begin{proof}
  The proof relies on Corollary~\ref{th:csg}: for all $t,\tp\in[0,1]$,
  \begin{align*}
    D(P_t,P_\tp)
    &= \sup_{f\in\calQ}\left\{ \: \E[(1{-}t)P_0+tP_1]{f(x)}-\E[(1{-}\tp)P_0+\tp P_1]{f(x)} \: \right\} \\
    &= \sup_{f\in\calQ}\left\{ \: -(t{-}\tp)\E[P_0]{f(x)}+(t{-}\tp)\E[P_1](f(x)) \: \right\} \\
    &= |t{-}\tp| \sup_{f\in\calQ}\left\{ \: \E[P_0]{f(x)} - \E[P_1]{f(x)} \right\}~.
  \end{align*}
  where the last equality relies on the fact that if $f\in\calQ$, then $-f\in\calQ$.
  By Corollary~\ref{th:csg}, the mixture curve is a constant speed minimal geodesic.
  Since this is true for any $P_0,P_1\in\calM$ such that $D(P_0,P_1)<\infty$,
  the distance $D$ is strictly intrinsic.
\end{proof}

\begin{remark}
  \label{rem:polishproperties}
  Although Theorem~\ref{th:mixtureipm} makes $(\calM,D)$ a length
  space, it does not alone make it a strictly intrinsic Polish space.
  One also needs to establish the completeness and
  separability\footnote{\relax For instance the set of probability
    measures on $\R$ equipped with the total variation
    distance~\eqref{eq:totalvariation} is not separable because any
    dense subset needs one element in each of the disjoint balls
    $B_x=\{\,P{\in}\pspace{\R}:D_{TV}(P,\delta_x){<}1/2\,\}$.}
  properties of a Polish space. Fortunately, these properties are true
  for both $(\calM^1,W_1)$ and $(\calM^1,\calE_d)$ when the ground
  space is Polish.\footnote{\relax For the Wasserstein distance, see
    \cite[Theorem~6.18]{villani-2009}.  For the Energy distance, both
    properties can be derived from Theorem~\ref{th:edmmd} after
    recaling that $\Phi_\calX\subset\calH$ is both complete and
    separable because it is isometric to $\calX$ which is Polish.}
\end{remark}

Since both the $1$-Wasserstein distance $W_1$ and the Energy Distance
or MMD $\calE_d$ belong to the IPM family, $\calM$ equipped with
either distance is a strictly intrinsic Polish space. Any two
probability measures are connected by at least one minimal geodesic,
the mixture geodesic. We shall see later that the $1$-Wasserstein
distance admits many more minimal geodesics. However, in the case of
ED/MMD distances, mixture geodesics are the only minimal geodesics.

\begin{theorem}
  \label{th:mmdgeodesics}
  Let $K$ be a characteristic kernel and let $\calM$ be equipped with
  the MMD distance $\calE_{d_K}$. Then any two probability measures
  $P_0,P_1\in\calM$ such that $\calE_{d_K}(P_0,P_1)<\infty$ are joined
  by exactly one constant speed minimal geodesic, the mixture
  geodesic~\eqref{eq:mixture}.
\end{theorem}

Note that $\calE_{d_K}$ is also the ED for the strongly negative definite kernel $d_K$.

\begin{proof}
  \def\ce{\calE_{d_K}}
  Theorem~\ref{th:mixtureipm} already shows that any two
  measures $P_0,P_1\in\calM$ are connected by the mixture geodesic $P_t$.
  We only need to show that it is unique.
  For any $t\in[0,1]$, the measure $P_t$ belongs to the set
  \begin{equation}
    \label{eq:condp}
    \big\{\: P\in\calM: \ce(P_0,P)=t D
          \text{~and~} \ce(P,P_1)=(1{-}t) D \:\big\}\subset\calM
  \end{equation}
  where $D=\ce(P_0,P_1)$. Thanks to Theorem~\ref{th:edmmd},
  $\E[P_t]{\Phi_x}$ must belong to the set
  \begin{equation}
    \label{eq:condh}
    \big\{\: \Psi\in\calH: \|\E[P_0]{\Phi_x}-\Psi\|\subH = tD
    \text{~and~} \|\Psi-\E[P_1]{\Phi_x}\|\subH = (1{-}t)D \:\big\}\subset\calH~.
  \end{equation}
  with $D=\|\E[P_0]{\Phi_x}-\E[P_1]{\Phi_x}\|\subH$.
  Since there is only one point $\Psi$ that satisfies these conditions in $\calH$,
  and since Corollary~\ref{th:characteristic} says that the map $P\mapsto\E[P]{\Phi_x}$
  is injective, there can only be one $P$ satisfying \eqref{eq:condp} and this
  must be $P_t$. Therefore the mixture geodesic is the only one.
\end{proof}

\subsection{Displacement geodesics}
\label{sec:displacementgeodesics}

\paragraph{Displacement geodesics in the Euclidean case~}

Let us first assume that $\calX$ is a Euclidean space and $\calM$ is
equipped with the $p$-Wasserstein distance $W_p$.  Let
$P_0,P_1\in\calM^p$ be two distributions with optimal transport plan
$\pi$. The \emph{displacement curve} joining $P_0$ to $P_1$ is formed
by the distributions
\begin{equation}
  \label{eq:euclideandisplacement}
  \forall t\in[0,1] \quad P_t = \big((1{-}t)x+ty\big)\pfd\pi(x,y) ~.
\end{equation}
Intuitively, whenever the optimal transport plan specifies that a
grain of probability mass must be transported from $x$ to $y$ in
$\calX$, we follow the shortest path connecting $x$ and $y$, that
is, in a Euclidean space, a straight line, but we drop the grain after
performing a fraction $t$ of the journey.

\begin{proposition}
  \label{th:euclideandisplacementgeodesic}
  Let $\calX$ be a Euclidean space and let $\calM$ be equipped with
  the $p$-Wasserstein distance \eqref{eq:wasserstein} for some
  $p\geq1$. Any displacement curve~\eqref{eq:euclideandisplacement}
  joining two distributions $P_0,P_1$ such that $W_p(P_0,P_1)<\infty$
  is a constant speed minimal geodesic, making $W_p$ a strictly
  intrinsic distance.
\end{proposition}

\begin{proof}
  Let $\pi_{01}$ be the optimal transport plan between $P_0$ and $P_1$.
  For all $t,\tp\in[0,1]$, define a tentative transport plan $\pi_{t\tp}$
  between $P_t$ and $P_\tp$ as
  \[
     \pi_{t\tp} = \big(\,(1{-}t)x+ty,\,(1{-}\tp)x+\tp y\,\big)\pfd\pi_{01}(x,y) ~\in~\Pi(P_t,P_\tp)~.
  \]
  Then
  \begin{align*}
    W_p(P_t,P_\tp)^p &\leq \E[(x,y)\sim\pi_{t\tp}]{\:\|x-y\|^p\:} \\
    &= \E[(x,y)\sim\pi]{\:\|(1{-}t)x+ty-(1{-}\tp)x-\tp y\|^p\:} \\
    &= |t-\tp|^p \: \E[(x,y)\sim\pi]{\:\|x-y\|\:} = |t-\tp|^p\:W_p(P_0,P_1)^p~.
  \end{align*}
  By Corollary~\ref{th:csg}, the displacement curve is a constant
  speed minimal geodesic.  Since this is true for any
  $P_0,P_1\in\calM$ such that $W_p(P_0,P_1)<\infty$, the distance
  $W_p$ is strictly intrinsic.
\end{proof}

When $p>1$, it is a well-known fact that the displacement geodesics
are the only geodesics of $\calM$ equipped with the $W_p$ distance.

\begin{proposition}
  \label{th:euclideanwpgeodesics}
  The displacement geodesics \eqref{eq:euclideandisplacement} are the only
  constant speed minimal geodesics of $\calM$ equipped with
  the $p$-Wasserstein distance $W_p$ with $p>1$.
\end{proposition}
This is a good opportunity to introduce a very useful lemma.
\begin{lemma}[Gluing]
  \label{th:gluing}
  Let $\calX_i$, $i=1,2,3$ be Polish metric spaces.
  Let probability measures $\mu_{12}\in\pspace{\calX_1{\times}\calX_2}$ and
  $\mu_{23}\in\pspace{\calX_2{\times}\calX_3}$ have the same marginal distribution $\mu_2$ on $\calX_2$.
  Then there exists $\mu\in\pspace{\calX_1{\times}\calX_2\times\calX_3}$ such that
  $(x,y)\pfd\mu(x,y,z)=\mu_{12}$ and $(y,z)\pfd\mu(x,y,z)=\mu_{23}$.
\end{lemma}
\begin{proofx}{Proof notes for Lemma~\ref{th:gluing}}
  At first sight, this is simply $P(x,y,z)=P(x|y)P(z|y)P(y)$ with
  $\mu_{12}{=}P(x,y)$, $\mu_{23}{=}P(y,z)$. Significant technical
  difficulties arise when $P(y)=0$. This is where one needs the
  topological properties of a Polish space \cite{berti-2015}.
\end{proofx}

\begin{proofx}{Proof of Proposition~\ref{th:euclideanwpgeodesics}}
  Let $t\in[0,1]\mapsto P_t$ be a constant speed minimal geodesic.
  Any point $P_t$ must satisfy the equality
  \[
  W_p(P_0,P_t)+W_p(P_t,P_1)=W_p(P_0,P_1)
  \]
  Let $\pi_0$ and $\pi_1$ be the optimal transport plans
  associated with $W_p(P_0,P_t)$ and $W_p(P_t,P_1)$ and construct
  $\pi_3\in\pspace{\calX^3}$ by gluing them. Then we must have
  \begin{multline*}
    \left(\E[(x,y,z)\sim\pi_3]{\,\|x-y\|^p\,}\right)^{1/p} +
    \left(\E[(x,y,z)\sim\pi_3]{\,\|y-z\|^p\,}\right)^{1/p} \\
    ~=~ W_p(P_0,P_1)
    ~\leq~ \left(\E[(x,y,z)\sim\pi_3]{\,\|x-z\|^p\,}\right)^{\tfrac1p}~.
  \end{multline*}
  Thanks to the properties of the Minkowski's inequality, this can
  only happen for $p>1$ if there exists $\lambda\in[0,1]$ such that,
  $\pi_3$-almost surely, $\|x{-}y\|{=}\lambda\|x{-}z\|$ and
  $\|y{-}z\|{=}(1{-}\lambda)\|x{-}z\|$.  This constant can only be $t$
  because $W_P(P_0,P_t)=tW_p(P_0,P_1)$ on a constant speed minimal
  geodesic. Therefore $y=tx+(1{-}t)y$, $\pi_3$-almost surely.
  Therefore $P_t=y\pfd\pi(x,y,z)$ describes a displacement curve
  as defined in~\eqref{eq:euclideandisplacement}.  \qedsymbol
\end{proofx}

Note however that the displacement geodesics are not the only minimal
geodesics of the $1$-Wasserstein distance $W_1$. Since $W_1$ is an IPM
\eqref{eq:emdual}, we know that the mixture geodesics are also minimal
geodesics (Theorem~\ref{th:mixtureipm}).  There are in fact many more
geodesics. Intuitively, whenever the optimal transport plan from $P_0$
to $P_1$ transports a grain of probability from $x$ to $y$, we can
drop the grain after a fraction $t$ of the journey (displacement
geodesics), we can randomly decide whether to transport the grain as
planned (mixture geodesics), we can also smear the grain of
probability along the shortest path connecting $x$ to $y$, and we can
do all of the above using different $t$ in different parts of the
space.

\medskip
\paragraph{Displacement geodesics in the general case~}

The rest of this section reformulates these results to the more
general situation where $\calX$ is a strictly intrinsic Polish space.
Rather than following the random curve approach described in
\cite[Chapter~7]{villani-2009}, we chose a more elementary approach
because we also want to characterize the many geodesics of $W_1$. Our
definition is equivalent for $p{>}1$ and subtly weaker for~$p{=}1$.

The main difficulties are that we may no longer have a single shortest
path connecting two points $x,y\in\calX$, and that we may not be able
to use the push-forward formulation \eqref{eq:euclideandisplacement}
because the function that returns the point located at position $t$
along a constant speed minimal geodesic joining $x$ to $y$ may not
satisfy the necessary measurability requirements.

\begin{definition}[Displacement geodesic]
  \label{def:displacementgeodesics}
  Let $\calX$ be a strictly intrinsic Polish metric space and let
  $\calM^p$ be equipped with the $p$-Wasserstein distance $W_p$. The
  curve $t\in[0,1]\mapsto P_t\in\calM^p$ is called a displacement
  geodesic if, for all $0{\leq}t{\leq}\tp{\leq}1$, there is a
  distribution $\pi_4\in\pspace{\calX^4}$ such that
  \begin{itemize}[nosep]
  \item[$i)$]
    The four marginals of $\pi_4$ are respectively
    equal to $P_0$, $P_t$, $P_\tp$, $P_1$.
  \item[$ii)$]
    The pairwise marginal $(x,z)\pfd\pi_4(x,u,v,z)$ is
    an optimal transport plan
    \[ W_p(P_0,P_1)^p = \E[(x,u,v,z)\sim\pi_4]{d(x,z)^p}~. \]
  \item[$iii)$]
    The following relations hold $\pi_4(x,u,v,z)$-almost surely\/:
    \[
    d(x,u)=t\,d(x,z) \,,\quad
    d(u,v)=(\tp-t)\,d(x,z) \,,\quad
    d(v,z)=(1{-}\tp)\,d(x,z)~.
    \]
  \end{itemize}
\end{definition}

\begin{proposition}
  \label{th:displacementgeodesics}
  Definition~\ref{def:displacementgeodesics} indeed implies that $P_t$
  is a constant speed minimal geodesic of length
  $W_p(P_0,P_1)$. Furthermore, for all $0\leq t\leq\tp\leq1$, all
  the pairwise marginals of $\pi_4$ are optimal transport plans
  between their marginals.
\end{proposition}
\begin{proof}
  For all $0\leq t\leq\tp\leq1$, we have
  \begin{align*}
    W_p(P_t,P_\tp)^p
    &\leq \E[(x,u,v,z)\sim\pi_4]{d(u,v)^p}\\
    &= (t{-}\tp)^p\,\E[(x,u,v,z)\sim\pi_4]{d(x,z)^p}
    = (t{-}\tp)^p\,W_p(P_0,P_1)^p\,.
  \end{align*}
  By Corollary~\ref{th:csg}, the curve $P_t$ is a
  constant speed minimal geodesic.
  We can then write
  \begin{align*}
    \tp\,W_p(P_0,P_1) &= W_p(P_0,P_\tp)
       \leq \left( \E[(x,u,v,z)\sim\pi_4]{d(x,v)^p} \right)^{1/p}\\
      &\leq \left( \E[(x,u,v,z)\sim\pi_4]{(d(x,u)+d(u,v))^p} \right)^{1/p}\\
      &\leq \left( \E[(x,u,v,z)\sim\pi_4]{d(x,u)^p} \right)^{1/p}
            + \left( \E[(x,u,v,z)\sim\pi_4]{d(u,v)^p} \right)^{1/p}\\
      &\leq t\,\left( \E[(x,u,v,z)\sim\pi_4]{d(x,z)^p} \right)^{1/p}
            + (\tp{-}t)\,\left( \E[(x,u,v,z)\sim\pi_4]{d(x,z)^p} \right)^{1/p}\\
      &= \tp\,W_p(P_0,P_1)~,
  \end{align*}
  where the third inequality is Minkowski's inequality.  Since both
  ends of this chain of inequalities are equal, these inequalities
  must be equalities, implying that $(x,v)\pfd\pi_4$ is an optimal
  transport plan between $P_0$ and $P_\tp$.  We can do likewise for
  all pairwise marginals of $\pi_4$.
\end{proof}

The proposition above does not establish that a displacement geodesic
always exists. As far as we know, this cannot be established without
making an additional assumption such as the local compacity of the
intrinsic Polish space~$\calX$. Since it is often easy to directly
define a displacement geodesic as shown in
\eqref{eq:euclideandisplacement}, we omit the lengthy general proof.

\begin{theorem}
  \label{th:wpgeodesics}
  Let $\calX$ be a strictly intrinsic Polish metric space and let
  $P_0,P_1$ be two distributions of $\calM^p$ equipped with the
  $p$-Wasserstein with $p>1$. The only constant speed minimal
  geodesics of length $W_p(P_0,P_1)$ joining $P_0$ and $P_1$ are the
  displacement geodesics.
\end{theorem}

\begin{proof}
  Let $P_t$ be a constant speed minimal geodesic of length $W_p(P_0,P_1)$.
  By Theorem~\ref{th:twopoints}, for all $0\leq t\leq\tp\leq1$,
  \[
     W_p(P_0,P_t)+W_p(P_t,P_\tp)+W_p(P_\tp,P_1) = W_p(P_0,P_1)~.
  \]
  Let $\pi_4$ be constructed by gluing optimal transport plans
  associated with the three distances appearing on the left
  hand side of the above equality. We can then write
  \begin{align*}
    W_p(P_0,P_1) &\leq \left(\E[(x,u,v,z)\sim\pi_4]{d(x,z)^p}\right)^{1/p}\\
    &\leq \left(\E[(x,u,v,z)\sim\pi_4]{(d(x,u)+d(u,v)+d(v,z))^p}\right)^{1/p}\\
    &\leq \left(\E[(x,u,v,z)\sim\pi_4]{d(x,u)^p}\right)^{1/p}
    + \left(\E[(x,u,v,z)\sim\pi_4]{d(u,v)^p}\right)^{1/p}\\
    & \hskip8em + \left(\E[(x,u,v,z)\sim\pi_4]{d(v,z)^p}\right)^{1/p}\\
    &= W_p(P_0,P_t)+W_p(P_t,P_\tp)+W_p(P_\tp,P_1) = W_p(P_0,P_1)~.
  \end{align*}
  Since this chain of inequalities has the same value in both ends,
  all these inequalities must be equalities. The first one means that
  $\pi_4$ is an optimal transport plan for $W_p(P_0,P_1)$. The second
  one means that $(d(x,u)+d(u,v)+d(v,z)=d(x,z)$, $\pi_4$-almost
  surely. When $p>1$, the third one, Minkowski's inequality can only
  be an inequality if there are scalars
  $\lambda_1{+}\lambda_2{+}\lambda_3{=}1$ such that, $\pi_4$-almost
  surely, $d(x,u)=\lambda_1d(x,z)$, $d(x,u)=\lambda_2d(x,z)$, and
  $d(v,z)=\lambda_3d(x,z)$. Since $P_t$ must satisfy
  Corollary~\ref{th:csg}, these scalars can only be $\lambda_1=t$,
  $\lambda_2=\tp{-}t$, and $\lambda_3=1{-}\tp$.
\end{proof}

\paragraph{Minimal geodesics for the $1$-Wasserstein distance~}
\medskip

We can characterize the many minimal geodesics of
the $1$-Wasserstein distance using a comparable strategy.

\begin{theorem}
  \label{th:w1geodesics}
  Let $\calX$ be a strictly intrinsic Polish space and
  let $\calM^1$ be equipped with the distance~$W_1$.
  A curve $t\in[a,b]\mapsto P_t\in\calM^1$ joining
  $P_a$ and $P_b$ is a minimal geodesic of length
  $W_1(P_a,P_b)$ if and only if, for all $a\leq t\leq\tp\leq b$,
  there is a distribution $\pi_4\in\pspace{\calX^4}$ such that
  \begin{itemize}[nosep]
  \item[$i)$]
    The four marginals of $\pi_4$ are respectively
    equal to $P_a$, $P_t$, $P_\tp$, $P_b$.
  \item[$ii)$]
    The pairwise marginal $(x,z)\pfd\pi_4(x,u,v,z)$ is
    an optimal transport plan
    \[ W_p(P_a,P_b) = \E[(x,u,v,z)\sim\pi_4]{d(x,z)}~. \]
  \item[$iii)$]
    The following relation holds $\pi_4(x,u,v,z)$-almost surely\/:
    \[
       d(x,u)+d(u,v)+d(v,z) = d(x,z)~.
    \]
  \end{itemize}
\end{theorem}

It is interesting to compare this condition to
Theorem~\ref{th:twopoints}.  Instead of telling us that two
successive triangular inequalities in the probability space $(\calM^1,W_1)$
must be an equality, this result tells us that the same holds
almost-surely in the sample space $(\calX,d)$.  In particular, this
means that $x$, $u$, $v$, and $z$ must be aligned along a geodesic of
$\calX$. In the case of a mixture geodesic, $u$ and $v$ coincide with
$x$ or $z$. In the case of a displacement geodesic, $u$ and $v$ must
be located at precise positions along a constant speed geodesic
joining $x$ to $z$. But there are many other ways to fulfil these
conditions.

\begin{proof}
  When $P_t$ is a minimal geodesic, Theorem~\ref{th:twopoints} states
  \[
     \forall a\leq t\leq\tp\leq b \quad
     W_1(P_a,P_t)+W_1(P_t,P_\tp)+W_1(P_\tp,P_b) = W_1(P_a,P_b)~.
  \]
  Let $\pi_4$ be constructed by gluing optimal transport plans
  associated with the three distances appearing on the left
  hand side of the above equality. We can then write
  \begin{align*}
    \E[(x,u,v,z)\sim\pi_4]{d(x,z)} 
    &\leq \E[(x,u,v,z)\sim\pi_4]{d(x,u)+d(u,v)+d(v,z)} \\
    &= W_1(P_a,P_b) \leq \E[(x,u,v,z)\sim\pi_4]{d(x,z)}~. 
  \end{align*}
  Since this chain of equalities has the same value on both ends,
  all these inequalities must be equalities. The first one means
  that $d(x,u)+d(u,v)+d(v,z) = d(x,z)$, $\pi_4$-almost surely.
  The second one means that $(x,z)\pfd\pi_4$ is an optimal transport plan.

  Conversely, assume $P_t$ satisfies the conditions listed in the proposition.
  We can then write, for all $a\leq t\leq\tp\leq b$,
  \begin{align*}
    W_1(P_a,P_b) &= \E[(x,u,v,z)\sim\pi_4]{d(x,z)} \\
    &= \E[(x,u,v,z)\sim\pi_4]{d(x,u)+d(u,v)+d(v,z)} \\
    &= W_1(P_a,P_t)+W_1(P_t,P_\tp)+W_1(P_\tp,P_b) ~,
  \end{align*}
  and we conclude using Theorem~\ref{th:twopoints}.
\end{proof}


\section{Unsupervised learning and geodesic structures}
\label{sec:convexity}

We have seen in the previous section that
the geometry of the space $\calM$ of probability distributions
changes considerably with our choice of a probability distance.
Critical aspects of these possible geometries can be
understood from the characterization of the shortest paths
between any two distributions:
\begin{itemize}
\item
  With the Energy Distance~$\calE_{d}$ or the Maximum Mean
  Discrepancy~$\calE_{d_K}$, the sole shortest path is the mixture
  geodesic~(Theorem~\ref{th:mmdgeodesics}.)
\item
  With the $p$-Wasserstein distance $W_p$, for $p>1$, the sole
  shortest paths are displacement
  geodesics~(Theorem~\ref{th:wpgeodesics}.)
\item
  With the $1$-Wasserstein distance $W_1$, there are many shortest
  paths, including the mixture geodesic, all the displacement geodesics,
  and all kinds of hybrid curves~(Theorem~\ref{th:w1geodesics}.)
\end{itemize}

The purpose of this section is to investigate the consequences of
these geometrical differences on unsupervised learning problems. In
the following discussion, $Q\in\calM$ represents the data distribution
which is only known through the training examples, and
$\calF\subset\calM$ represent the family of parametric models
$P_\theta\in\calM$ considered by our learning algorithm.

Minimal geodesics in length spaces can sometimes be compared to line
segments in Euclidean spaces because both represent shortest paths
between two points. This association provides the means to extend the
familiar Euclidean notion of convexity to length spaces.  This section
investigates the geometry of implicit modeling learning problems
through the lens of this generalized notion of convexity.

\subsection{Convexity \emph{\`{a}-la-carte}}

We now assume that $\calM$ is a strictly intrinsic Polish space
equipped with a distance $D$. Let $\calC$ be a family of smooth
constant speed curves in $\calM$. Although these curves need not be
minimal geodesics, the focus of this section is limited to three
families of curves defined in Section~\ref{sec:geodesics}:
\begin{itemize}[nosep]
\item the family $\calC_g(D)$ of all minimal geodesics in $(\calM,D)$,
\item the family $\calC_d(W_p)$ of the displacement geodesics in $(\calM^p,W_p)$,
\item the family $\calC_m$ of the mixture curves in $\calM$.
\end{itemize}

\begin{definition}
  Let $\calM$ be a strictly intrinsic Polish space.  A closed subset
  $\calF\subset\calM$ is called convex with respect to the family of
  curves $\calC$ when, for all $P_0,P_1\in\calM$, the set~$\calC$
  contains a curve $t\in[0,1]\mapsto
  P_t\in\calX$ connecting $P_0$ and $P_1$ whose graph is contained in
  $\calF$, that is, $P_t\in\calF$ for all $t\in[0,1]$.
\end{definition}

\begin{definition}
  Let $\calM$ be a strictly intrinsic Polish space.  A real-valued
  function $f$ defined on $\calM$ is called convex with respect to the
  family of constant speed curves $\calC$ when, for every curve
  $t\in[0,1]\mapsto P_t\in\calM$ in $\calC$, the function
  $t\in[0,1]\mapsto f(P_t)\in\R$ is convex.
\end{definition}

For brevity we also say that $\calF$ or $f$ is \emph{geodesically convex}
when $\calC=\calC_g(D)$, \emph{mixture convex} when $\calC=\calC_m$,
and \emph{displacement convex} when $\calC=\calC_d(W_p)$.

\begin{theorem}[Convex optimization \emph{\`{a}-la-carte}]
  \label{th:convexopt}
  Let $\calM$ be a strictly intrinsic Polish space equipped with a
  distance $D$. Let the closed subset $\calF\subset\calM$ and the cost
  function $f:\calM\mapsto\R$ be both convex with respect to a same
  family $\calC$ of constant speed curves.  Then, for all $M\geq\min_\calF(f)$,
  \begin{itemize}[nosep]
  \item[$i)$] the level set $L(f,\calF,M)=\{P\in\calF:f(P)\leq M\}$ is
    connected,
  \item[$ii)$] for all $P_0\in\calF$ such that $f(P_0)>M$ and all
    $\epsilon>0$, there exists $P\in\calF$ such that
    $D(P,P_0)=\mathcal{O}(\epsilon)$ and $f(P)\leq
    f(P_0)-\epsilon(f(P_0){-}M)$.
  \end{itemize}
\end{theorem}

This result essentially means that it is possible to optimize the cost
function $f$ over $\calF$ with a descent algorithm. Result $(i)$ means
that all minima are global minima, and result $(ii)$ means that any
neighborhood of a suboptimal distribution $P_0$ contains a
distribution $P$ with a sufficiently smaller cost to ensure that the
descent will continue.

\begin{proof}
  $(i)$:~ Let $P_0,P_1\in L(f,\calF,M)$. Since they both belong to
  $\calF$, $\calC$ contains a curve $t\in[0,1]\mapsto P_t\in\calF$
  joining $P_0$ and $P_1$.  For all $t\in[0,1]$, we know that
  $P_t\in\calF$ and, since $t\mapsto f(P_t)$ is a convex function, we
  can write $f(P_t)\leq(1-t)f(P_0)+tf(P_1)\leq M$.  Therefore $P_t\in
  L(f,\calF,M)$. Since this holds for all $P_0,P_1$, $L(f,\calF,M)$ is
  connected.

  $(ii)$:~ Let $P_1\in L(f,\calF,M)$. Since $\calF$ is convex
  with respect to $\calC$, $\calC$ contains a constant speed curve
  $t\in[0,1]\mapsto P_t\in\calF$ joining $P_0$ and $P_1$.
  Since this is a constant speed curve,
  \( d(P_0,P_\epsilon) \leq \epsilon D(P_0,P_1)~, \)
  and since $t\mapsto f(P_t)$ is convex,
  \( f(P_\epsilon) \leq (1-\epsilon)f(P_0)+\epsilon f(P_1)\),
  implies \( f(P_\epsilon\leq f(P_0) -\epsilon(f(P_0)-M)\,.\)
\end{proof}

One particularly striking aspect of this result is that it does not
depend on the parametrization of the family $\calF$. Whether the cost
function $C(\theta)=f(G_\theta\pfd\mu_z)$ is convex or not is
irrelevant: as long as the family $\calF$ and the cost function $f$
are convex with respect to a well-chosen set of curves, the level sets
of the cost function $C(\theta)$ will be connected, and there will be
a nonincreasing path connecting any starting point $\theta_0$ to a
global optimum~$\theta^*$.

\smallskip
It is therefore important to understand how the definition of $\calC$
makes it easy or hard to ensure that both the model family $\calF$
and the training criterion $f$ are convex with respect to~$\calC$.

\subsection{The convexity of implicit model families}

We are particularly interested in the case of implicit models
(Section~\ref{sec:implicitmodeling}) in which the
distributions~$P_\theta$ are expressed by pushing the samples
$z\in\mathcal{Z}$ of a known source distribution
$\mu_z\in\pspace{\mathcal{Z}}$ through a parametrized generator
function $G_\theta(z)\in\calX$. This push-forward operation defines a
deterministic coupling between the distributions~$\mu_z$
and~$P_\theta$ because the function~$G_\theta$ maps every source
sample~$z\in\calZ$ to a single point $G_\theta(z)$ in $\calX$. In
contrast, a stochastic coupling distribution
$\pi_\theta\in\Pi(\mu_z,P_\theta)\subset\pspace{\calZ\times\calX}$
would be allowed to distribute a source sample~$z\in\calZ$ to several
locations in $\calX$, according to the conditional distribution
$\pi_\theta(x|z)$.

The deterministic nature of this coupling makes it very hard to
achieve mixture convexity using smooth generator functions~$G_\theta$.

\begin{example}
  \label{ex:nomixture}
  Let the distributions $P_0,P_1\in\calF$ associated with parameters
  $\theta_0$ and $\theta_1$ have disjoint supports separated by a
  distance greater than $D{>}0$. Is there a continuous path
  $t\in[0,1]\mapsto\theta_t$ in parameter space such that
  $G_{\theta_t}\pfd\mu_z$ is the mixture $P_t=(1{-}t)P_0+P_1$~?
  
  If we assume there is such a path, we can write
  \[
    \mu_z\{ G_{\theta_0}(z) \in \text{supp}(P_0) \} = 1
  \]
  and, for any $\epsilon{>}0$,
  \[
    \mu_z\{ G_{\theta_\epsilon}(z) \in \text{supp}(P_1) \} = \epsilon>0~.
  \]
  Therefore, for all $\epsilon{>}0$, there exists $z{\in}\calZ$ such that
  $d(G_{\theta_0}(z),G_{\theta_\epsilon}(z)){\geq}D$. Clearly such a
  generator function is not compatible with the smoothness
  requirements of an efficient learning algorithm.
\end{example}

In contrast, keeping the source sample $z$ constant, a small change of
the parameter $\theta$ causes a small displacement of the generated
sample~$G_\theta(z)$ in the space $\calX$. Therefore we can expect
that such an implicit model family has a particular affinity for
displacement geodesics.

\smallskip

It is difficult to fully assess the consequences of the
quasi-impossibility to achieve mixture convexity with implicit models.
For instance, although the Energy Distance $\calE_d(Q,P_\theta)$ is a
mixture convex function (see Proposition~\ref{th:mixconvipm} in the
next section), we cannot expect that a family $\calF$ of implicit
models will be mixture convex.

\begin{figure}
  \centering
  \includegraphics[width=0.48\linewidth]{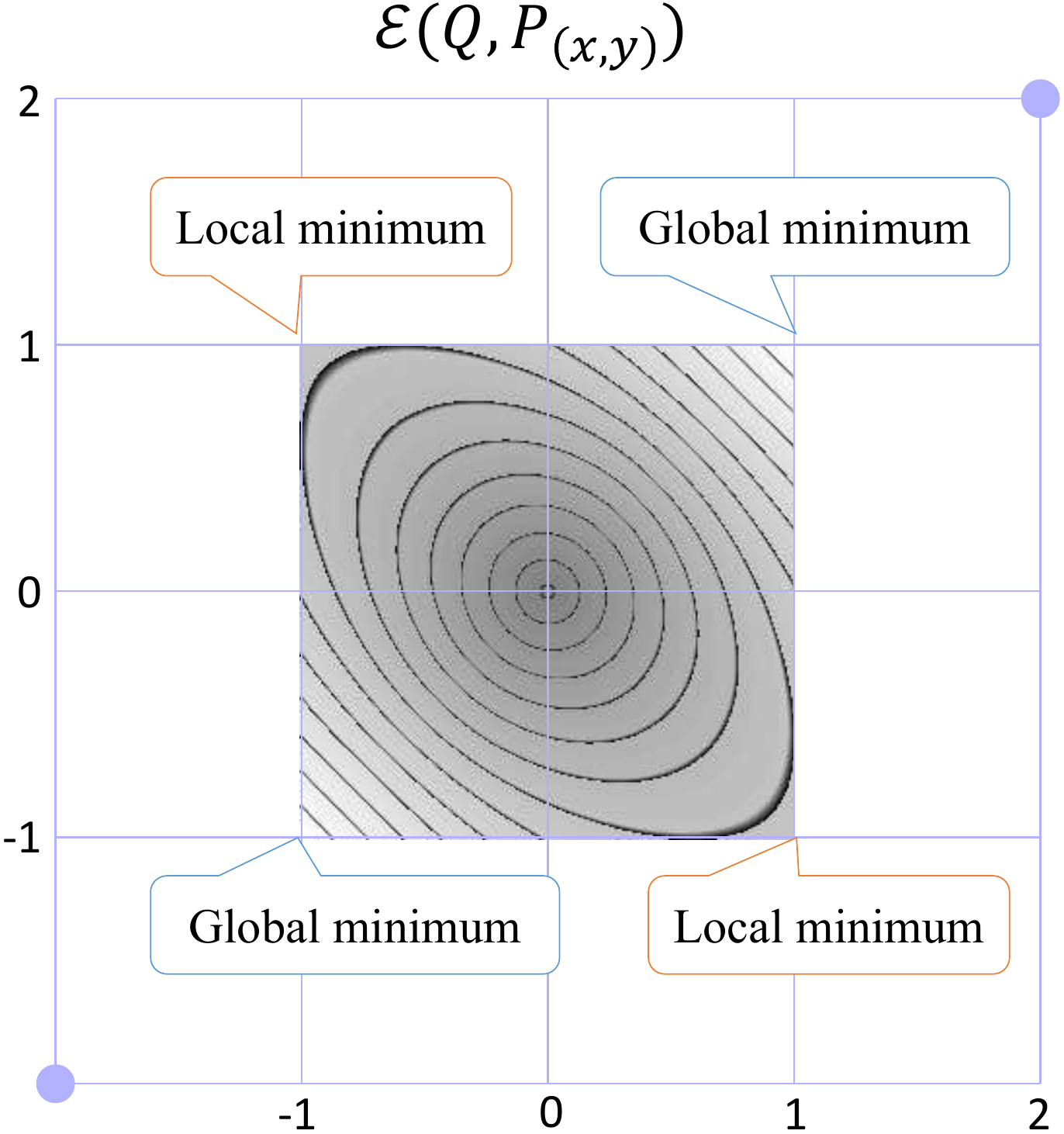}
  \quad
  \includegraphics[width=0.48\linewidth]{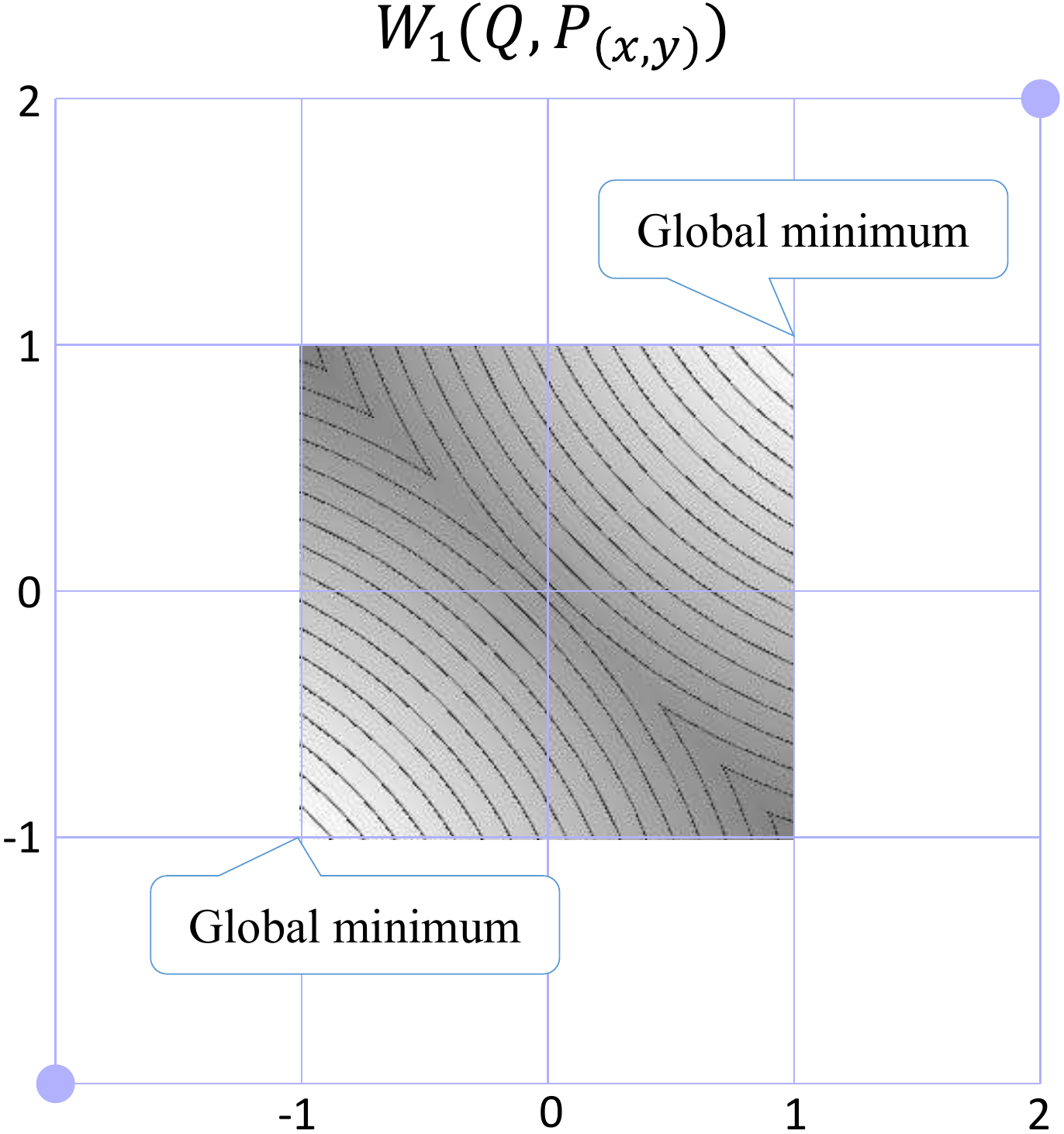}
  \caption{\label{fig:counter}
    Level sets for the problems described in Example~\ref{ex:counter}.}
\end{figure}

\begin{example}
  \label{ex:counter}
  Let $\mu_z$ be the uniform distribution on $\{-1,+1\}$.
  Let the parameter $\theta$ be constrained to the square $[-1,1]^2\subset\R^2$
  and let the generator function be
  \[
       G_\theta: z\in\{-1,1\}\mapsto G_\theta(z)=z\theta~.
  \]
  The corresponding model family is
  \[
  \calF = \big\{ P_{\theta} = \tfrac12(\delta_\theta{+}\delta_{-\theta}) :
    \theta\in[-1,1]\times[-1,1] \big\}~.
  \]
  It is easy to see that this model family is displacement convex but
  not mixture convex. Figure~\ref{fig:counter} shows the level sets
  for both criteria $\calE(Q,P_\theta)$ and $W_1(Q,P_\theta)$ for the
  target distribution $Q=P_{(2,2)}\notin\calF$.  Both criteria have
  the same global minima in $(1,1)$ and $(-1,-1)$.  However the energy
  distance has spurious local minima in $(-1,1)$ and $(1,-1)$ with a
  relatively high value of the cost function.
\end{example}

Constructing such an example in $\R^2$ is nontrivial. Whether such
situations arise commonly in higher dimension is not known. However we
empirically observe that the optimization of a MMD criterion on
high-dimensional image data often stops with unsatisfactory
results~(Section~\ref{sec:martin}).

\subsection{The convexity of distances}

Let $Q\in\calM$ be the target distribution for our learning algorithm.
This could be the true data distribution or the empirical training set
distribution. The learning algorithm minimizes the cost function
\begin{equation}
  \label{eq:cdopt}
  \min_{P_\theta\in\calF} C(\theta) ~\stackrel{\Delta}{=}~ D(Q,P_\theta)\,.
\end{equation}
The cost function itself is therefore a distance. Since such a
distance function is always convex in a Euclidean space, we can ask
whether a distance in a strictly intrinsic Polish space is
geodesically convex. This is not always the
case. Figure~\ref{fig:convexl1} gives a simple counter-example in
$\R^2$ equipped with the $L_1$ distance.

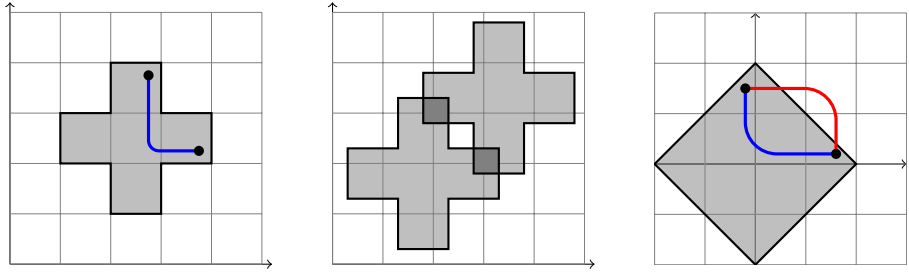
\begin{figure}
  \bigskip
  \centering
  \def\cross{%
    (3,2) -- (4,2) -- (4,3) -- (5,3) -- 
    (5,4) -- (4,4) -- (4,5) -- (3,5) -- 
    (3,4) -- (2,4) -- (2,3) -- (3,3) -- cycle}
  \begin{tikzpicture}[scale=0.67]
    \draw [->] (1,1) -- (1,6.2);
    \draw [->] (1,1) -- (6.2,1);
    \path [fill=lightgray] \cross;
    \draw [help lines] (1,1) grid (6,6);
    \draw [thick] \cross;
    \draw [rounded corners,very thick,blue] (3.75,4.75) |- (4.75,3.25);
    \path [fill=black] (3.75,4.75) circle [radius=0.1];
    \path [fill=black] (4.75,3.25) circle [radius=0.1];
  \end{tikzpicture}
  \qquad
  \begin{tikzpicture}[scale=0.67]
    \def\sha{[xshift=-0.7cm,yshift=-0.7cm]}
    \def\shb{[xshift=0.8cm,yshift=0.8cm]}
    \draw [->] (1,1) -- (1,6.2);
    \draw [->] (1,1) -- (6.2,1);
    \path \sha [fill=lightgray] \cross;
    \path \shb [fill=lightgray] \cross;
    \begin{scope}
      \clip \sha \cross;
      \path \shb [fill=gray] \cross;
    \end{scope}
    \draw [help lines] (1,1) grid (6,6);
    \draw \sha [thick] \cross;
    \draw \shb [thick] \cross;
  \end{tikzpicture}
  \qquad
  \begin{tikzpicture}[scale=0.67]
    \def\ball{(-2,0) -- (0,2) -- (2,0) -- (0,-2) -- cycle}
    \clip (-2,-2) rectangle (3,3);
    \path [fill=lightgray] \ball;
    \draw [->] (-2,0) -- (3,0);
    \draw [->] (0,-2) -- (0,3);
    \draw [help lines] (-2,-2) grid (3,3);
    \draw [thick] \ball;
    \node (n1) at (1.6,0.2) {};
    \node (n2) at (-0.2,1.5) {};
    \draw [rounded corners=12,very thick,blue] (n1.base) -| node[near end](bnode){} (n2.base);
    \draw [rounded corners=12,very thick,red] (n1.base) |- node[midway](rnode){} (n2.base);
    \path [fill=black] (n1) circle [radius=0.1];
    \path [fill=black] (n2) circle [radius=0.1];
  \end{tikzpicture}
  \caption{\label{fig:convexl1} Geodesic convexity often differs from
    Euclidean convexity in important ways. There are many different
    minimal geodesics connecting any two points in $\R^2$ equipped
    with the $L_1$ distance (see also
    Figure~\ref{fig:l1geodesics}). The cross-shaped subset of $\R_2$
    shown in the left plot is geodesically convex. The center plot
    shows that the intersection of two geodesically convex sets is not
    necessarily convex or even connected. The right plot shows that
    two points located inside the unit ball can be connected by a
    minimal geodesic that does not stay in the unit ball. This means
    that the $L_1$ distance itself is not convex because its
    restriction to that minimal geodesic is not convex.}
\end{figure}

Yet we can give a positive answer for the mixture convexity of IPM distances.

\begin{proposition}
  \label{th:mixconvipm}
  Let $\calM$ be equipped with a distance $D$ that belongs to the IPM family~\eqref{eq:ipmform}.
  Then $D$ is mixture convex.
\end{proposition}
\begin{proof}
  Let $t\in[0,1]\mapsto P_t=(1{-}t)P_0+tP_1$ be a mixture curve.
  Theorem~\ref{th:mixtureipm} tells us that such mixtures are minimal geodesics.
  For any target distribution $Q$ we can write
  \begin{align*} 
    D(Q,P_t) &= {\sup}_{f\in\calQ}\left\{\: \E[Q]{f(x)} - \E[P_t]{f(x)} \:\right\} \\
    &= {\sup}_{f\in\calQ}\left\{\: (1{-}t)\,\big(\E[Q]{f(x)}-\E[P_0]{f(x)}\big)
                               + t\,\big(\E[Q]{f(x)}-\E[P_1]{f(x)}\big)\:\right\} \\
    &\leq (1{-}t)\,D(Q,P_0) + t\,D(Q,P_1)~.
  \end{align*}                             
  The same holds for any segment $t\in[t_1,t_2]\subset[0,1]$ because
  such segments are also mixture curves up to an affine
  reparametrization. Therefore $t\mapsto D(Q,P_t)$ is convex.
\end{proof}

Therefore, when $D$ is an IPM distance, and when $\calF$ is a mixture
convex family of generative models, Theorem~\ref{th:convexopt} tells
us that a simple descent algorithm can find the global minimum
of~\eqref{eq:cdopt}. As discussed in Example~\ref{ex:nomixture}, it is
very hard to achieve mixture convexity with a family of implicit
models. But this could be achieved with nonparametric techniques.

\smallskip

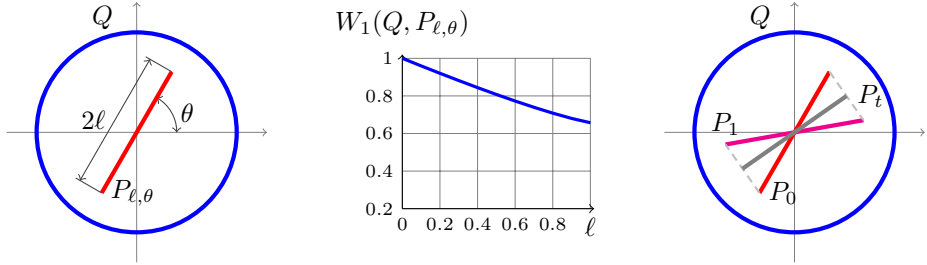
\begin{figure}
  \begin{tikzpicture}[scale=1.33]
    \draw [->,help lines] (-1.3,0) -- (1.3,0);
    \draw [->,help lines] (0,-1.3) -- (0,1.3);
    \draw [ultra thick,blue] (0,0) circle (1);
	\node [above] at (110:1) {$Q$};
	\draw [<->,darkgray] (0:0.4) arc(0:60:0.4) 
	    node[right,black,midway]{$\theta$};
	\begin{scope}[rotate=60]
	 \node [right] at (-0.7,0) {$P_{\ell,\theta}$};
	 \draw [darkgray] (-0.7,0) -- +(0,.3);
	 \draw [darkgray] (0.7,0) -- +(0,.3);
	 \draw [<->,darkgray] (-0.7,0.25) -- node[left,black]{$2\ell$} (0.7,0.25);
	 \draw [ultra thick,red] (-0.7,0) -- (0.7,0);
	\end{scope}
  \end{tikzpicture}
  \qquad
  \begin{tikzpicture}[scale=.5]
    \draw [help lines] (0,1) grid (5,5);
    \draw [->] (0,1) -- (5.1,1);
    \draw [->] (0,1) -- (0,5.1);
    \begin{scope}[scale=5]
      \draw [very thick,blue] 
      (0,1) -- (0.1,0.9598) -- (0.2,0.9204) -- 
      (0.3,0.8818) -- (0.4,0.8442) -- (0.5,0.8078) -- 
      (0.6,0.7727) -- (0.7,0.7394) -- (0.8,0.7084) -- 
      (0.9,0.6803) -- (1,0.6573);
    \end{scope}
    \node [above] at (0,5.3) {$W_1(Q,P_{\ell,\theta})$};
    \node [below] at (5,1) {$\ell$};
    \foreach \x in {0.2,0.4,0.6,0.8,1}
    \node[left] at (0,\x*5) {$\scriptstyle\x$};
    \foreach \x in {0,0.2,0.4,0.6,0.8}
    \node[below] at (\x*5,1) {$\scriptstyle\x$};
    \node at (0,-0.2) {~};
  \end{tikzpicture}
  \qquad
  \begin{tikzpicture}[scale=1.33]
    \draw [->,help lines] (-1.3,0) -- (1.3,0);
    \draw [->,help lines] (0,-1.3) -- (0,1.3);
    \draw [ultra thick,blue] (0,0) circle (1);
    \node [above] at (110:1) {$Q$};
    \begin{scope}[rotate=60]
      \node [right] at (-0.7,0) {$P_0$};
      \draw [ultra thick,red] (-0.7,0) -- (0.7,0);
    \end{scope}
    \begin{scope}[rotate=10]
      \node [above] at (-0.7,0) {$P_1$};
      \draw [ultra thick,magenta] (-0.7,0) -- (0.7,0);
    \end{scope}
    \draw[lightgray,densely dashed,thick] (60:0.7) -- (10:0.7);
    \draw[lightgray,densely dashed,thick] (60:-0.7) -- (10:-0.7);	
    \draw[ultra thick,gray] 
    (intersection of 60:-0.7--10:-0.7 and 0,0--35:-1) --
    (intersection of 60:0.7--10:0.7 and 0,0--35:1)
    node[right,black] {$P_t$};
  \end{tikzpicture}
  \caption{\label{fig:nonconvex} Example~\ref{ex:notdispconvw1}
    considers a target distribution $Q$ that is uniform on the $\R^2$ unit circle
    and a displacement geodesic between two line segments centered on the origin
    and with identical length (right plot.)}
\end{figure}

However the same does not hold for displacement convexity.
For instance, the Wasserstein distance is not displacement convex,
even when the sample space distance $d$ is geodesically convex,
and even when the sample space is Euclidean.

\begin{example}
  \label{ex:notdispconvw1}
  Let $\calX$ be $\R^2$ equipped with the Euclidean distance.  Let $Q$
  be the uniform distribution on the unit circle, and let
  $P_{\ell,\theta}$ be the uniform distribution on a line segment of
  length $2\ell$ centered on the origin (Figure~\ref{fig:nonconvex},
  left). The distance $W_1(Q,P_{\ell,\theta})$ is independent on
  $\theta$ and decreases when $\ell\in[0,1]$ increases
  (Figure~\ref{fig:nonconvex}, center).  Consider a displacement
  geodesic $t\in[0,1]\mapsto P_t$ where $P_0=P_{\ell,\theta_0}$ and
  $P_1=P_{\ell,\theta_1}$ for $0{<}\theta_0{<}\theta_1{<}\pi/2$. Since
  the space $\R_2$ is Euclidean, displacements occur along straight
  lines. Therefore the distribution $P_t$ for $0{<}t{<}1$ is uniform
  on a slightly shorter line segment (Figure~\ref{fig:nonconvex},
  right), implying
  \[
     W_1(Q,P_t) ~>~ W_1(Q,P_0)=W_1(Q,P_1)~.
  \]
  Therefore the distance function $P\mapsto W_1(Q,P)$
  is not displacement convex.
\end{example}

\smallskip

Although this negative result prevents us from invoking
Theorem~\ref{th:convexopt} for the minimization of the Wasserstein
distance, observe that the convexity violation in
Example~\ref{ex:notdispconvw1} is rather small. Convexity violation
examples are in fact rather difficult to construct. The following
section shows that we can still obtain interesting guarantees by
bounding the size of the convexity violation.

\subsection{Almost-convexity}

We consider in this section that the distance $d$ is geodesically
convex in $\calX$: for any point $x\in\calX$ and any
constant speed geodesic $t\in[0,1]\mapsto\gamma_t\in\calX$,
\[
   d(x,\gamma_t)\leq (1{-}t)\,d(x,\gamma_0) + t\,d(x,\gamma_1)~.
\]
This requirement is of course verified when $\calX$ is an Euclidean
space. This is also trivially true when $\calX$ is a Riemannian or
Alexandrov space with nonpositive curvature~\cite{burago-2001}.

The following result bounds the convexity violation:

\begin{proposition}
  \label{th:almostconvexbound}
  Let $\calX$ be a strictly intrinsic Polish space equipped with a
  geodesically convex distance $d$ and let $\calM^1$ be equipped with
  the $1$-Wasserstein distance $W_1$. For all $Q\in\calM$ and all
  displacement geodesics $t\in[0,1]\mapsto P_t$,
  \[
     \forall t\in[0,1]\quad
     W_1(Q,P_t) \leq (1{-}t)\,W_1(Q,P_0) + t\,W_1(Q,P_1) + 2 t(1-t) K(Q,P_0,P_1)\,
  \]
  with 
  \(\displaystyle
      K(Q,P_0,P_1) \leq 2 \min_{u_0\in\calX} \E[u\sim Q]{d(u,u_0)}~.
  \)
\end{proposition}

\begin{figure}
  \centering
  \begin{tikzpicture}[thick,>=latex]
    \node [circle,draw,minimum width=1.5cm,thin] (x) at (1,3.5) {$x\sim P_0$};
    \node [circle,draw,minimum width=1.5cm,thin] (y) at (4,3.5) {$y\sim P_t$};
    \node [circle,draw,minimum width=1.5cm,thin] (z) at (7,3.5) {$z\sim P_1$};
    \node [circle,draw,minimum width=1.5cm,thin] (ux) at (2,1) {$u_x\sim Q$}; 
    \node [circle,draw,minimum width=1.5cm,thin] (u) at (4,1) {$u\sim Q$};
    \node [circle,draw,minimum width=1.5cm,thin] (uz) at (6,1)  {$u_z\sim Q$};
    \draw [<->,blue] (x) -- (y);	
    \draw [<->,blue] (y) -- (z);
    \draw [<->,blue] (x) .. node [below=1ex]{$\pi_3(x,y,z)$} controls (4,5.3) .. (z);
    \draw [<->,purple] (x) -- node[left]{$\pi_0$} (ux);
    \draw [<->,purple] (z) -- node[left]{$\pi_1$} (uz);
    \draw [->,olive,densely dashed] (ux) -- ++(0.5,1)
    -- node[above]{$[\,1{-}t\,]$} ++(1,0) -- (u);
    \draw [->,olive,densely dashed] (uz) -- ++(-0.5,1)
    -- node[above]{$[\,t\,]$} ++(-1,0) -- (u);
  \end{tikzpicture}
  \caption{\label{fig:glueplan} The construction of $\pi\in\pspace{\calX^6}$ in
    the proof of Proposition~\ref{th:almostconvexbound}.}
\end{figure}
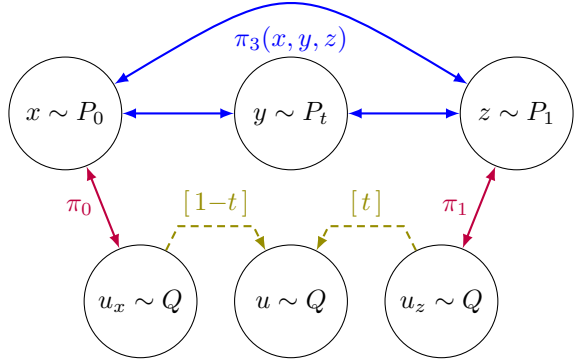

\begin{proof}
  The proof starts with the construction of a distribution
  $\pi\in\pspace{\calX^6}$ illustrated in Figure~\ref{fig:glueplan}.
  Thanks to Proposition~\ref{th:displacementgeodesics} we can
  construct a distribution $\pi_3(x,y,z)\in\pspace{\calX^3}$ whose
  marginals are respectively $P_0$, $P_t$, $P_1$, whose
  pairwise marginals are optimal transport plans, and such that,
  $\pi_3$-almost surely,
  \[
     d(x,y)=t\,d(x,z) \quad d(y,z)=(1{-}t)\,d(x,z)\,.
  \]
  We then construct distribution $\pi_5(x,y,z,u_x,u_z)$ by gluing
  $\pi_3$ with the optimal transport plans $\pi_0$ between $P_0$ and
  $Q$ and $\pi_1$ between $P_1$ and $Q$.  Finally
  $\pi(x,y,z,u_x,u_z,u)$ is constructed by letting $u$ be equal to
  $u_x$ with probability $1{-}t$ and equal to $u_z$ with probability
  $t$. The last three marginals of $\pi$ are all equal to $Q$.
  
  Thanks to the convexity of $d$ in $\calX$, the following
  inequalities hold $\pi$-almost surely:
  \begin{align*}
    d(u_x,y) &\leq (1{-}t)\,d(u_x,x)+t\,d(u_x,z)\\
             &\leq (1{-}t)\,d(u_x,x)+t\,d(u_z,z)+t\,d(u_x,u_z) \\
    d(u_z,y) &\leq (1{-}t)\,d(u_z,x)+t\,d(u_z,z)\\
             &\leq (1{-}t)\,d(u_x,x)+t\,d(u_z,z)+(1-t)\,d(u_x,u_z)~.
  \end{align*}
  Therefore
  \begin{align*}
    W_1(Q,P_t)
    &\leq \E[\pi]{d(u,y)}\\
    &= \E[\pi]{(1{-}t)d(u_x,y)+td(u_z,y)}\\
    &\leq \E[\pi]{ (1{-}t)\,d(u_x,x)+t\,d(u_z,z) + 2t(1-t)\,d(u_x,u_z) }\\
    &= (1{-}t)\,W_1(Q,P_0)+t\,W_1(Q,P_1)+ 2t(1-t)\,\E[\pi]{d(u_x,u_z)}~.
  \end{align*}
  For any $u_0\in\calX$, the constant $K$ in the last term
  can then be coarsely bounded with
  \begin{align*}
    K(Q,P_0,P_1) &= \E[\pi]{d(u_x,u_z)} \\
    &\leq \E[\pi]{d(u_x,u_0)}+\E[\pi]{d(u_0,u_z)}
    = 2 \E[u\sim Q]{d(u,u_0)}~.
  \end{align*}
  Taking the minimum over $u_0$ gives the final result.
\end{proof}

When the optimal transport plan from $P_0$ to $P_1$ specifies that a
grain of probability must be transported from $x$ to~$z$, its optimal
coupling counterpart in $Q$ moves from $u_x$ to~$u_z$. Therefore the
quantity $K(Q,P_0,P_1)$ quantifies how much the transport plan from
$P_t$ to $Q$ changes when $P_t$ moves along the geodesic. This idea
could be used to define a Lipschitz-like property such as
\[
  \forall P_0,P_1\in\calF_L\subset\calF\qquad
     K(Q,P_0,P_1) \leq L W_1(P_0,P_1)~.
\]
Clearly such a property does not hold when the transport plan changes
very suddenly. This only happens in the vicinity of
distributions~$P_t$ that can be coupled with $Q$ using multiple
transport plans.

Unfortunately we have not found an elegant way to leverage this idea
into a global description of the cost landscape.
Proposition~\ref{th:almostconvexbound} merely bounds~$K(Q,P_0,P_1)$ by
the expected diameter of the distribution~$Q$. We can nevertheless use
this bound to describe some level sets of $W_1(Q,P_\theta)$

\begin{theorem}
  \label{th:almostconvexopt}
  Let $\calX$ be a strictly intrinsic Polish space equipped with a
  geodesically convex distance $d$ and let $\calM^1$ be equipped with
  the $1$-Wasserstein distance $W_1$. Let $\calF\subset\calM^1$
  be displacement convex and let $Q\in\calM^1$ have expected diameter
  \[
      D = 2 \min_{u_0\in\calX}\E[u\sim Q]{d(u,u_0)}~.
  \] 
  Then the level set $L(Q,\calF,M)=\{P_\theta\in\calF:W_1(Q,P_\theta)\leq M\}$
  is connected if
  \[
      M > \inf_{P_\theta\in\calF}W_1(Q,P_\theta) + 2 D~.
  \]
\end{theorem}
\begin{proof}
  Choose $P_1\in\calF$ such that $W_1(Q,P_1)<M-2D$. For any
  $P_0,P_0^\prime\in L(Q,\calF,M)$, let $t\in[0,1]\mapsto P_t\in\calF$
  be a displacement geodesic joining $P_0$ and $P_1$ without
  leaving~$\calF$. Thanks to Proposition~\ref{th:almostconvexbound},
  \[
    W_1(Q,P_t) \leq (1{-}t)\,M + t\,(M-2D) + 2t(1-t)D = M - 2 t^2 D \leq M~.
  \]
  Therefore this displacement geodesic is contained in $L(Q,\calF,M)$
  and joins $P_0$ to $P_1$. We can similarly construct a second
  displacement geodesic that joins $P_0^\prime$ to $P_1$ without
  leaving $L(Q,\calF,M)$.  Therefore there is a continuous path
  connecting $P_0$ to $P_0^\prime$ without leaving $L(Q,\calF,M)$.
\end{proof}

This result means that optimizing the Wasserstein distance with a
descent algorithm will not stop before finding a generative model
$P\in\calF$ whose distance $W_1(Q,P)$ to the target distribution is
within $2D$ of the global minimum. Beyond that point, the algorithm
could meet local minima and stop progressing. Because we use a rather
coarse bound on the constant $K(Q,P_0,P_1)$, we believe that it is
possible to give much better suboptimality guarantee in particular
cases.

Note that this result does not depend on the parametrization of
$G_\theta$ and therefore applies to the level sets of potentially very
nonconvex neural network parametrizations. Previous results on the
connexity of such level sets
\cite{auffinger-benarous-2013,freeman-bruna-2016} are very tied to a
specific parametric form. The fact that we can give such a result in
an abstract setup is rather surprising. We hope that further
developments will clarify how much our approach can help these
efforts.

Finally, comparing this result with Example~\ref{ex:arora} also reveals a
fascinating possibility: a simple descent algorithm might in fact be
unable to find that the Dirac distribution at the center of the
sphere is a global minimum. Therefore the effective statistical
performance of the learning process may be subtantially better than
what Theorem~\ref{th:fournier} suggests. Further research is necessary
to check whether such a phenomenon occurs in practice.

\section{Conclusion}

This work illustrates how the geometrical study of probability
distances provides useful ---but still incomplete--- insights on the
practical performance of implicit modeling approaches using different
distances. In addition, using a technique that differs substantially
from previous works, we also obtain surprising global optimization
results that remain valid when the parametrization is nonconvex.


\section*{Acknowledgments}

We would like to thank Joan Bruna, Marco Cuturi, Arthur Gretton, Yann
Ollivier, and Arthur Szlam for stimulating discussions and also for
pointing out numerous related works.

\ifllncs
  \bibliographystyle{splncs03}
\else
  \bibliographystyle{plain}
\fi
\bibliography{geometry}

\begin{thebibliography}{10}

\bibitem{aizerman-1964}
M.~A. Aizerman, {\'E.}~M. Braverman, and L.~I. Rozono\'er.
\newblock Theoretical foundations of the potential function method in pattern
  recognition learning.
\newblock {\em Automation and Remote Control}, 25:821--837, 1964.

\bibitem{amari-2007}
Shun-ichi Amari and Hiroshi Nagaoka.
\newblock {\em Methods of information geometry}, volume 191.
\newblock American Mathematical Society, 2007.

\bibitem{arjovsky-chintala-bottou-2017}
Martin Arjovsky, Soumith Chintala, and L\'{e}on Bottou.
\newblock {Wasserstein} generative adversarial networks.
\newblock In {\em Proceedings of the 34nd International Conference on Machine
  Learning, {ICML} 2017, Sydney, Australia, 7-9 August, 2017}, 2017.

\bibitem{aronszajn-1950}
N.~Aronszajn.
\newblock Theory of reproducing kernels.
\newblock {\em Transactions of the American Mathematical Society}, 68:337--404,
  1950.

\bibitem{arora-2017}
Sanjeev Arora, Rong Ge, Yingyu Liang, Tengyu Ma, and Yi~Zhang.
\newblock Generalization and equilibrium in generative adversarial nets (gans).
\newblock {\em arXiv preprint arXiv:1703.00573}, 2017.

\bibitem{auffinger-benarous-2013}
Antonio Auffinger and G\'{e}rard {Ben Arous}.
\newblock Complexity of random smooth functions of many variables.
\newblock {\em Annals of Probability}, 41(6):4214--4247, 2013.

\bibitem{bellemare-2017}
Marc~G Bellemare, Ivo Danihelka, Will Dabney, Shakir Mohamed, Balaji
  Lakshminarayanan, Stephan Hoyer, and R{\'e}mi Munos.
\newblock The {Cramer} distance as a solution to biased {Wasserstein}
  gradients.
\newblock {\em arXiv preprint arXiv:1705.10743}, 2017.

\bibitem{berti-2015}
Patrizia Berti, Luca Pratelli, Pietro Rigo, et~al.
\newblock Gluing lemmas and skorohod representations.
\newblock {\em Electronic Communications in Probability}, 20, 2015.

\bibitem{borkar-1997}
Vivek~S Borkar.
\newblock Stochastic approximation with two time scales.
\newblock {\em Systems \& Control Letters}, 29(5):291--294, 1997.

\bibitem{bottou-curtis-nocedal-2016}
L\'eon Bottou, Frank~E. Curtis, and Jorge Nocedal.
\newblock Optimization methods for large-scale machine learning.
\newblock {\em CoRR}, abs/1606.04838, 2016.

\bibitem{bouchacourt-2016}
Diane Bouchacourt, Pawan~K Mudigonda, and Sebastian Nowozin.
\newblock {DISCO} nets: {DISsimilarity COefficients Networks}.
\newblock In {\em Advances in Neural Information Processing Systems 29}, pages
  352--360, 2016.

\bibitem{burago-2001}
Dmitri Burago, Yuri Burago, and Sergei Ivanov.
\newblock {\em A Course in Metric Geometry}.
\newblock volume 33 of AMS Graduate Studies in Mathematics. American
  Mathematical Society, 2001.

\bibitem{challis-barber-2012}
Edward Challis and David Barber.
\newblock Affine independent variational inference.
\newblock In F.~Pereira, C.~J.~C. Burges, L.~Bottou, and K.~Q. Weinberger,
  editors, {\em Advances in Neural Information Processing Systems 25}, pages
  2186--2194. Curran Associates, Inc., 2012.

\bibitem{cramer-1946}
Harald Cram\'{e}r.
\newblock {\em Mathematical Methods of Statistics}.
\newblock Princeton University Press, 1946.

\bibitem{denton-2015}
Emily Denton, Soumith Chintala, Arthur Szlam, and Rob Fergus.
\newblock Deep generative image models using a laplacian pyramid of adversarial
  networks.
\newblock In C.~Cortes, N.~D. Lawrence, D.~D. Lee, M.~Sugiyama, and R.~Garnett,
  editors, {\em Advances in Neural Information Processing Systems 28}, pages
  1486--1494. Curran Associates, Inc., 2015.

\bibitem{dereich-2013}
Steffen Dereich, Michael Scheutzow, and Reik Schottstedt.
\newblock Constructive quantization: approximation by empirical measures.
\newblock {\em Annales de l'I.H.P. Probabilités et statistiques},
  49(4):1183--1203, 2013.

\bibitem{dziu-2015}
Gintare~Karolina Dziugaite, Daniel~M. Roy, and Zoubin Ghahramani.
\newblock Training generative neural networks via maximum mean discrepancy
  optimization.
\newblock In {\em Proceedings of the Thirty-First Conference on Uncertainty in
  Artificial Intelligence, {UAI} 2015}, pages 258--267, 2015.

\bibitem{fournier-guillin-2015}
Nicolas Fournier and Arnaud Guillin.
\newblock On the rate of convergence in {Wasserstein} distance of the empirical
  measure.
\newblock {\em Probability Theory and Related Fields}, 162(3):707--738, Aug
  2015.

\bibitem{freeman-bruna-2016}
C~Daniel Freeman and Joan Bruna.
\newblock Topology and geometry of half-rectified network optimization.
\newblock {\em arXiv preprint arXiv:1611.01540}, 2016.

\bibitem{goodfellow-2014}
Ian~J. Goodfellow, Jean Pouget-Abadie, Mehdi Mirza, Bing Xu, David
  Warde-Farley, Sherjil Ozair, Aaron Courville, and Yoshua Bengio.
\newblock Generative adversarial nets.
\newblock In {\em Advances in Neural Information Processing Systems 27}, pages
  2672--2680. Curran Associates, Inc., 2014.

\bibitem{gretton-2012}
Arthur Gretton, Karsten~M. Borgwardt, Malte~J. Rasch, Bernhard Sch\"{o}lkopf,
  and Alexander Smola.
\newblock A kernel two-sample test.
\newblock {\em Journal of Machine Learning Research}, 13:723--773, 2012.

\bibitem{gulrajani-2017}
Ishaan Gulrajani, Faruk Ahmed, Martin Arjovsky, Vincent Dumoulin, and Aaron
  Courville.
\newblock Improved training of {Wasserstein} {GAN}s.
\newblock {\em arXiv preprint arXiv:1704.00028}, 2017.

\bibitem{hammersley-1950}
J.~M. Hammersley.
\newblock The distribution of distance in a hypersphere.
\newblock {\em The Annals of Mathematical Statistics}, 21(3):447--452, 1950.

\bibitem{hastie-tibshirani-friedman-2009}
T.~Hastie, R.~Tibshirani, and J.~Friedman.
\newblock {\em The Elements of Statistical Learning, Second Edition}.
\newblock Springer Series in Statistics. Springer Verlag, New York, 2009.

\bibitem{karras-2017}
Tero Karras, Timo Aila, Samuli Laine, and Jaakko Lehtinen.
\newblock Progressive growing of gans for improved quality, stability, and
  variation.
\newblock {\em arXiv preprint arXiv:1710.10196}, 2017.

\bibitem{khinchin-1929}
Aleksandr~Y. Khinchin.
\newblock Sur la loi des grandes nombres.
\newblock {\em Comptes Rendus de l'Acad\'emie des Sciences}, 1929.

\bibitem{kingma-welling-2013}
Diederik~P. Kingma and Max Welling.
\newblock Auto-encoding variational bayes.
\newblock {\em CoRR}, abs/1312.6114, 2013.

\bibitem{kocaoglu-2017}
Murat Kocaoglu, Christopher Snyder, Alexandros~G. Dimakis, and Sriram
  Vishwanath.
\newblock Causalgan: Learning causal implicit generative models with
  adversarial training.
\newblock {\em arXiv preprint arXiv:1709.02023}, 2017.

\bibitem{konda-tsitsiklis-2004}
Vijay~R Konda and John~N Tsitsiklis.
\newblock Convergence rate of linear two-time-scale stochastic approximation.
\newblock {\em Annals of applied probability}, pages 796--819, 2004.

\bibitem{kulkarni-2015}
Tejas~D Kulkarni, Pushmeet Kohli, Joshua~B Tenenbaum, and Vikash Mansinghka.
\newblock Picture: A probabilistic programming language for scene perception.
\newblock In {\em Proceedings of the {IEEE} conference on {Computer Vision And
  Pattern Recognition}, {CVPR} 2015}, pages 4390--4399, 2015.

\bibitem{lee-nevatia-2004}
Mun~Wai Lee and Ramakant Nevatia.
\newblock Dynamic human pose estimation using {Markov Chain Monte Carlo}
  approach.
\newblock In {\em 7th {IEEE} Workshop on Applications of Computer Vision /
  {IEEE} Workshop on Motion and Video Computing ({WACV/MOTION} 2005)}, pages
  168--175, 2005.

\bibitem{li-2017}
Chun-Liang Li, Wei-Cheng Chang, Yu~Cheng, Yiming Yang, and Barnab{\'a}s
  P{\'o}czos.
\newblock Mmd gan: Towards deeper understanding of moment matching network.
\newblock {\em arXiv preprint arXiv:1705.08584}, 2017.

\bibitem{li-2015}
Yujia Li, Kevin Swersky, and Richard Zemel.
\newblock Generative moment matching networks.
\newblock In {\em Proceedings of the 32Nd International Conference on
  International Conference on Machine Learning - Volume 37}, ICML'15, pages
  1718--1727, 2015.

\bibitem{liu-bousquet-chaudhuri-2017}
Shuang Liu, Olivier Bousquet, and Kamalika Chaudhuri.
\newblock Approximation and convergence properties of generative adversarial
  learning.
\newblock {\em arXiv preprint arXiv:1705.08991}, 2017.
\newblock to appear in NIPS 2017.

\bibitem{milgrom-segal-2002}
Paul Milgrom and Ilya Segal.
\newblock Envelope theorems for arbitrary choice sets.
\newblock {\em Econometrica}, 70(2):583--601, 2002.

\bibitem{mueller-1997}
Alfred M\"uller.
\newblock Integral probability metrics and their generating classes of
  functions.
\newblock {\em Advances in Applied Probability}, 29(2):429--443, 1997.

\bibitem{neal-2001}
Radford~M. Neal.
\newblock Annealed importance sampling.
\newblock {\em Statistics and Computing}, 11(2):125--139, Apr 2001.

\bibitem{nguyen-2010}
XuanLong Nguyen, Martin~J Wainwright, and Michael~I Jordan.
\newblock Estimating divergence functionals and the likelihood ratio by convex
  risk minimization.
\newblock {\em IEEE Transactions on Information Theory}, 56(11):5847--5861,
  2010.

\bibitem{nowozin-cseke-tomioka-2016}
Sebastian Nowozin, Botond Cseke, and Ryota Tomioka.
\newblock {f-GAN}: Training generative neural samplers using variational
  divergence minimization.
\newblock In {\em Advances in Neural Information Processing Systems 29}, pages
  271--279. 2016.

\bibitem{rachev-2013}
Svetlozar~T Rachev, Lev Klebanov, Stoyan~V Stoyanov, and Frank Fabozzi.
\newblock {\em The methods of distances in the theory of probability and
  statistics}.
\newblock Springer, 2013.

\bibitem{radford-2015}
Alec Radford, Luke Metz, and Soumith Chintala.
\newblock Unsupervised representation learning with deep convolutional
  generative adversarial networks.
\newblock {\em arXiv preprint arXiv:1511.06434}, 2015.

\bibitem{rezende-2014}
Danilo~Jimenez Rezende, Shakir Mohamed, and Daan Wierstra.
\newblock Stochastic backpropagation and approximate inference in deep
  generative models.
\newblock In {\em Proceedings of the 31th International Conference on Machine
  Learning, {ICML} 2014}, pages 1278--1286, 2014.

\bibitem{romaszko-2017}
Lukasz Romaszko, Christopher~KI Williams, Pol Moreno, and Pushmeet Kohli.
\newblock Vision-as-inverse-graphics: Obtaining a rich 3d explanation of a
  scene from a single image.
\newblock In {\em Proceedings of the {IEEE} {Conference on Computer Vision and
  Pattern Recognition}, {CVPR} 2017}, pages 851--859, 2017.

\bibitem{salimans-2016}
Tim Salimans, Ian Goodfellow, Wojciech Zaremba, Vicki Cheung, Alec Radford, and
  Xi~Chen.
\newblock Improved techniques for training {GAN}s.
\newblock In {\em Advances in Neural Information Processing Systems 29}, pages
  2234--2242, 2016.

\bibitem{schoenberg-1938}
Isaac~J. Schoenberg.
\newblock Metric spaces and positive definite functions.
\newblock {\em Transactions of the American Mathematical Society}, 44:522--536,
  1938.

\bibitem{schoelkopf-smola-2002}
Bernhard Sch\"{o}lkopf and Alexander~J. Smola.
\newblock {\em Learning with Kernels}.
\newblock {MIT} Press, Cambridge, MA, 2002.

\bibitem{sejdinovic-2013}
Dino Sejdinovic, Bharath Sriperumbudur, Arthur Gretton, and Kenji Fukumizu.
\newblock Equivalence of distance-based and rkhs-based statistics in hypothesis
  testing.
\newblock {\em The Annals of Statistics}, 41(5):2263--2291, 2013.

\bibitem{serfling-1980}
Robert~J. Serfling.
\newblock {\em Approximation theorems of mathematical statistics}.
\newblock John Wiley \& Sons, New York; Chichester, 1980.

\bibitem{sriperumbudur-2016}
Bharath Sriperumbudur.
\newblock On the optimal estimation of probability measures in weak and strong
  topologies.
\newblock {\em Bernoulli}, 22(3):1839--1893, 08 2016.

\bibitem{sriperumbudur-2012}
Bharath~K Sriperumbudur, Kenji Fukumizu, Arthur Gretton, Bernhard
  Sch{\"o}lkopf, Gert~RG Lanckriet, et~al.
\newblock On the empirical estimation of integral probability metrics.
\newblock {\em Electronic Journal of Statistics}, 6:1550--1599, 2012.

\bibitem{sriperumbudur-2011}
Bharath~K Sriperumbudur, Kenji Fukumizu, and Gert~RG Lanckriet.
\newblock Universality, characteristic kernels and rkhs embedding of measures.
\newblock {\em Journal of Machine Learning Research}, 12(Jul):2389--2410, 2011.

\bibitem{szekely-rizzo-2013}
G{\'a}bor~J Sz{\'e}kely and Maria~L Rizzo.
\newblock Energy statistics: A class of statistics based on distances.
\newblock {\em Journal of statistical planning and inference},
  143(8):1249--1272, 2013.

\bibitem{szekely-2002}
J.~G{\'a}bor Sz{\'e}kely.
\newblock E-statistics: The energy of statistical samples.
\newblock Technical Report 02-16, Bowling Green State University, Department of
  Mathematics and Statistics, 2002.

\bibitem{theis-2016}
Lucas Theis, Aaron van~den Oord, and Matthias Bethge.
\newblock A note on the evaluation of generative models.
\newblock In {\em International Conference on Learning Representations}, 2016.

\bibitem{villani-2009}
C\'edric Villani.
\newblock {\em Optimal Transport: Old and New}.
\newblock Grundlehren der mathematischen Wissenschaften. Springer, Berlin,
  2009.

\bibitem{vonmises-1947}
Richard von Mises.
\newblock On the asymptotic distribution of differentiable statistical
  functions.
\newblock {\em The Annals of Mathematical Statistics}, 18(3):309--348, 1947.

\bibitem{zinger-1992}
A.~A. Zinger, Ashot~V. Kakosyan, and Lev~B. Klebanov.
\newblock A characterization of distributions by mean values of statistics and
  certain probabilistic metrics.
\newblock {\em Journal of Soviet Mathematics}, 4(59):914--920, 1992.
\newblock Translated from \emph{Problemy Ustoichivosti Stokhasticheskikh
  Modelei-Trudi seminara}, 1989, pp 47-55.

\end{thebibliography}

\end{document}